\documentclass{article}

    \usepackage[preprint]{neurips_2023}

\usepackage[utf8]{inputenc} % allow utf-8 input
\usepackage[T1]{fontenc}    % use 8-bit T1 fonts
\usepackage{hyperref}       % hyperlinks
\usepackage{url}     
\usepackage{capt-of}% simple URL typesetting
\usepackage{booktabs}       % professional-quality tables
\usepackage{amsfonts}       % blackboard math symbols
\usepackage{nicefrac}       % compact symbols for 1/2, etc.
\usepackage{microtype}      % microtypography
\usepackage{xcolor}         % colors
\usepackage{amsmath}

\usepackage{float}

\usepackage{blindtext}
\usepackage{amssymb}
\usepackage{adjustbox}
\usepackage{wrapfig,lipsum,booktabs}
\usepackage{amsthm}
\usepackage{amsmath,amssymb}
\usepackage{booktabs}       % professional-quality tables
\usepackage{algpseudocode}
\usepackage{algorithm2e}
\usepackage[utf8]{inputenc} % allow utf-8 input
\usepackage[T1]{fontenc}    % use 8-bit T1 fonts
\usepackage{hyperref}       % hyperlinks
\usepackage{url}            % simple URL typesetting
\usepackage{booktabs}       % professional-quality tables
\usepackage{amsfonts}       % blackboard math symbols
\usepackage{nicefrac}       % compact symbols for 1/2, etc.
\usepackage{microtype}      % microtypography
\usepackage{xcolor}         % colors
\newcommand\norm[1]{\lVert#1\rVert}
\newenvironment{hproof}{%
  \proof}{\endproof}
\usepackage{amsmath}
\usepackage{mathtools}
\usepackage{graphicx}
\usepackage{subcaption}
\usepackage{cleveref}
\usepackage{thmtools}
\usepackage{thm-restate}
\usepackage{graphicx}
\usepackage{blkarray}
\usepackage{amsmath}
\usepackage{mathtools}
\usepackage{wrapfig}
\usepackage{graphicx}
\usepackage{caption} 
\declaretheorem[name=Theorem,numberwithin=section]{theorem}

\declaretheorem[name=Definition,numberwithin=section]{definition}

\declaretheorem[name=Assumption,numberwithin=section]{assumption}

\title{Generalization Bounds for MBP via Sparse Matrix Sketching}

\author{%
    Etash Kumar Guha \\
    College of Computing\\
    Georgia Institute of Technology\\
    Atlanta, GA, 30332 \\
  \texttt{etash@gatech.edu} \\
  \And
  Prasanjit Dubey \\
    School of Industrial and Systems Engineering\\
    Georgia Institute of Technology\\
    Atlanta, GA, 30332 \\
  \texttt{pdubey31@gatech.edu} \\
  \And
  Xiaoming Huo \\
    School of Industrial and Systems Engineering\\
    Georgia Institute of Technology\\
    Atlanta, GA, 30332 \\
  \texttt{huo@gatech.edu} \\
}

\begin{document}

\maketitle
\begin{abstract}
% 1. In the abstract, you summarize your work by talking about (i) the problem you want to solve; (ii) why it is important; (iii) how others deal with the same problem; (iv) what are the limitations of other works; (v) how do you solve such limitations and what is the novelty of your work; and (vi) are the results promising? what did you learn from the experiments? please be aware that the abstract has a limitation on the number of words.

In this paper, we derive a novel bound on the generalization error of overparameterized neural networks when they have undergone Magnitude-Based Pruning (MBP). Our work builds on the bounds in \citet{Arora2018}, where the error depends on one, the approximation induced by pruning, and two, the number of parameters in the pruned model, and improves upon standard norm-based generalization bounds. The pruned estimates obtained using MBP are close to the unpruned functions with high probability, which improves the first criteria. Using Sparse Matrix Sketching, the space of the pruned matrices can be efficiently represented in the space of dense matrices with much smaller dimensions, thereby improving the second criterion. This leads to stronger generalization bound than many state-of-the-art methods, thereby breaking new ground in the algorithm development for pruning and bounding generalization error of overparameterized models. Beyond this, we extend our results to obtain generalization bounds for Iterative Pruning \citep{Frankle2018}. We empirically verify the success of this new method on ReLU-activated Feed Forward Networks on the MNIST and CIFAR10 datasets.

\end{abstract}
\section{Introduction}

Overparameterized neural networks are often used in practice as they achieve remarkable generalization errors \citep{GoodBengCour16}. However, their immense size makes them slow and expensive to run during inference \citep{han2015learning}. Machine learning (ML) practitioners often employ \emph{Magnitude-Based pruning} (MBP) to amend this computational complexity. After training large neural networks, the parameters or matrix elements within the model with the smallest magnitude are set to $0$. This reduces the memory requirements of the model and the inference time greatly. However, MBP also has been shown to induce little generalization error and, in fact, often reduces the generalization error compared to the original model \citep{han2015learning, Li2016, Cheng2017}. Examining where and why strong generalization happens can help build learning algorithms and models that generalize better \citep{Foret2020, lei2018}. 
However, theoretical analyses of why MBP achieves strong generalization errors still need to be made. Providing such analyses is challenging for several reasons. First, removing the smallest weights is a relatively unstudied operation in linear algebra, and only few tools are available to analyze the properties of pruned matrices. Second, it is difficult to characterize the distribution of weights after training and pruning. 

However, \citet{Arora2018} provided a tool that more comprehensively analyzes the generalization error of models with fewer parameters effectively. Specifically, they upper bounded the generalization error of a large neural network when compressed. We can directly apply this result to pruned models as pruned models intrinsically have fewer parameters. Their bound is split into two parts: the amount of error introduced by the model via the compression and the number of different parameters in the compressed model. We use this primary tool to show that the outputs from MBP generalize well since they don't introduce much error and have fewer parameters. We prove both of these phenomena with a few simple assumptions. Namely, given some justifiable assumptions on the distribution of the trained weight parameters, we develop an upper bound of the amount of error the pruned neural network suffers with high probability. We also demonstrate that the number of parameters needed to fully express the space of pruned models is relatively small. Specifically, we show that our set of pruned matrices can be efficiently represented in the space of dense matrices of much smaller dimension via the Sparse Matrix Sketching. 

Combining the two parts of the bound, we get a novel generalization error bound that is competitive with state-of-the-art generalization bounds. Moreover, to our knowledge, this is the \emph{first} generalization bound for MBP that uses Compression Bounds. We empirically verify the success of our novel approach on the MNIST and CIFAR10 datasets where our bound is several orders of magnitude better (at least $10^7$ times better on CIFAR10, refer \Cref{fig:cifar}) than well-known standard bounds of \citet{neyshabur2015norm},\citet{Bartlett2017}, and \citet{Neyshabur2017}. We extend our framework to show that using Iterative Magnitude Pruning (IMP) or Lottery Tickets \citep{Frankle2018} also generalizes. Namely, \citet{Malach2020} shows that IMP produces results with small error and few nonzero parameters. We use matrix sketching to efficiently count the number of parameters in a usable way for our generalization bound. This results in a strong generalization bound that, to our knowledge, has only been analyzed empirically \citep{bartoldson2020generalization, jin2022pruning}.

\paragraph{Contributions} We formally list our contributions here. We first prove the error induced by MBP is small relative to the original model. Moreover, we demonstrate that our MBP achieves sufficient sparsity, i.e., relatively few nonzero parameters are left after pruning. To tighten our generalization bounds, we show that the pruned matrices from MBP can be sketched into smaller dense matrices. We combine the results above to prove that the generalization error of the pruned models is small. We extend the proof framework above to establish a generalization error bound for IMP. To our knowledge, these are the first results studying the generalization of pruned models through either MBP or IMP. We empirically verify that our generalization bounds improve upon many standard generalization error bounds for MLPs on CIFAR10 and MNIST datasets.

% \section{Idea Overview}
% Since \cite{han2015learning}, pruning learned models into sparser, more efficient models has observed very good generalization performance while also achieving efficiency improvements. While this idea has existed in practice for quite some time, few works have rigorously proved the connection between generalization and sparsity. Recently \cite{CompressionBounds} showed that you can bound the generalization error of a model by the number of parameters in the model. Intuitively, proving the generaliztion of sparse models by counting the number of parameters sounds natural. We demonstrate this idea in our paper.
\section{Related Works}
\subsection{Norm-based Generalization Bounds}
In recent years, many works have studied how to use parameter counting and weight norms to form tighter generalization bounds as an evolution from classical Rademacher Complexity and VC dimension. \citet{galanti2023norm} uses Rademacher Complexity to develop a generalization bound for naturally sparse networks such as those from sparse regularization. \citet{neyshabur2015norm} studies a general class of norm-based bounds for neural networks. Moreover, \citet{bartlett2002rademacher} used Rademacher and Gaussian Complexity to form generalization bounds. \citet{Long2020} gives generalization error bounds for Convolutional Neural Networks (CNNs) using the distance from initial weights and the number of parameters that are independent of the dimension of the feature map and the number of pixels in the input. \citet{daniely2019generalization} uses approximate description length as an intuitive form for parameter counting. 
\subsection{Pruning Techniques}
While MBP is one of the most common forms of pruning in practice, other forms exist. \citet{Collins2014} induce sparsity into their CNNs by using $\ell_1$ regularization in their training. \citet{Molchanov2017} develops iterative pruning frameworks for compressing deep CNNs using greedy criteria-based pruning based on the Taylor expansion and fine-tuning by backpropagation. \citet{liu2017learning} use Filter Sparsity alongside Network Slimming to enable speedups in their CNNs. \citet{ullrich2017soft} coins soft-weight sharing as a methodology of inducing sparsity into their bounds. Moreover, \citet{hooker2019compressed} empirically studied which samples of data-pruned models will significantly differ from the original models. Many works use less common pruning methods such as coresets \citep{Mussay2019} or the phenomenon of IMP \citep{Frankle2018, Malach2020}.

% While MBP is one of the most common forms of pruning in practice, there are other forms of pruning. \citet{lee2018snip} uses connection sensitivity to do pruning before training. Moreover, \citet{Collins2014} induce sparsity into their Convolutional Networks by regularization in their bounds. \citet{Ebrahimi2021} puts forward a compression algorithm for DNNs by directing the parameters of the pruned model towards a flatter solution in terms of Kronecker-factored Approximate Curvature block diagonal based approximation of the spectral radius of the Hessian. \citet{Molchanov2017} develops iterative pruning frameworks for compressing deep CNNs using greedy criteria-based pruning based on Taylor expansion along with fine tuning by backpropagation. \citet{liu2017learning} use Filter Sparsity alongside Network Slimming to enable speedups in their Convolutional networks. \citet{ullrich2017soft} coins soft-weight sharing as a methodology of inducing sparsity into their bounds. Moreover, \citet{hooker2019compressed} empirically studied on which samples of data-pruned models will significantly differ from the original models. Many works use less common pruning methods such as coresets \cite{Mussay2019} or the phenomenon of Iterative Magnitude Pruning \cite{Frankle2018, Malach2020}.
\section{Preliminary}
\subsection{Notation}
\label{sec:notation}
We consider a standard multiclass classification problem where for a given sample $x$, we predict the class $y$, which is an integer between 1 and $k$. We assume that our model uses a learning algorithm that generates a set of $L$ matrices $\mathbf{M} =\{\mathbf{A}_1,\dots, \mathbf{A}_L\}$ where $\mathbf{A}_i \in \mathbb{R}^{d_1^i \times d_2^i}$. Here, $d_1^i, d_2^i$ are the dimensions of the $i$th layer. Therefore, given some input $x$, the output of our model denoted as $\mathbf{M}(x)$ is defined as 
$$\mathbf{M}(x) = \mathbf{A}_L\phi_{L-1}(\mathbf{A}_{L-1}\phi_{L-2}(\dots \mathbf{A}_2\phi_1(\mathbf{A}_1x))),$$ thereby mapping $x$ to $\mathbf{M}(x) \in \mathbb{R}^{k}$. Here, $\phi_i$ is the activation function for the $i$th layer of $L_i$ Lipschitz-Smoothness. When not vague, we will use the notation $x^0 = x$ and $x^1 = \mathbf{A}_1x$ and $x^2 = \mathbf{A}_2\phi_1(\mathbf{A}_1x)$ and so on. Given any data distribution $\mathcal{D}$ the expected margin loss for some margin $\gamma>0$ is defined as $$R_{\gamma}(\mathbf{M})=\mathbb{P}_{(x, y) \sim \mathcal{D}}\left[\mathbf{M}(x)[y] \leq \gamma+\max _{j \neq y} \mathbf{M}(x)[j]\right]\text{.}$$ The population risk $R(\mathbf{M})$ is obtained as a special case of $R_{\gamma}(\mathbf{M})$ by setting $\gamma=0$. The empirical margin loss for a classifier is defined as $$\hat{R}_\gamma(\mathbf{M}) = \frac{1}{|\mathcal{S}|}\sum_{(x, y) \in \mathcal{S}} \mathbb{I}\left(\mathbf{M}(x)[y] - \underset{j \neq y}{\max}\left(\mathbf{M}(x)[j]\right) \geq \gamma\right),$$ for some margin $\gamma>0$ where $\mathcal{S}$ is the dataset provided (when $\gamma=0$, this becomes the classification loss).  Intuitively, $\hat{R}_\gamma(\mathbf{M})$ denotes the number of elements the classifier $\mathbf{M}$ predicts the correct $y$ with a margin greater than or equal to $\gamma$. Moreover, we define the size of $\mathcal{S}$ to be $|\mathcal{S}| = n$. We will denote $\hat{\mathbf{M}} = \{\hat{\mathbf{A}}^1,\dots, \hat{\mathbf{A}}^L\}$ as the compressed model obtained after pruning $\mathbf{M}$. The generalization error of the pruned model is then $R_0(\hat{\mathbf{M}})$. Moreover, we will define the difference matrix at layer $l$ as $\Delta^l = \mathbf{\mathbf{A}}^l - \hat{\mathbf{\mathbf{A}}}^l$. Now that we have formally defined our notation, we will briefly overview the main generalization tool throughout this paper. 
\subsection{Compression Bounds}
\label{sec:revarora}
As compression bounds are one of the main theoretical tools used throughout this paper, we will briefly overview the bounds presented in \citet{Arora2018}. Given that we feed a model $f$ into a compression algorithm, the set of possible outputs is a set of models $ G_{\mathcal{A},s}$ where $\mathcal{A}$ is a set of possible parameter configurations and $s$ is some starter information given to the compression algorithm. We will call $g_A$ as one such model corresponding to parameter configuration $A \in \mathcal{A}$. Moreover, if there exists a compressed model $g_A \in  G_{\mathcal{A},s}$ such that for any input in a dataset $\mathcal{S}$, the outputs from $g_A$ and $f$ differ by at most $\gamma$, we say  $f$ is $(\gamma, \mathcal{S})$ compressible. Formally, we make this explicit in the following definition.
\begin{definition}
    If $f$ is a classifier and $G_{\mathcal{A},s} = \{g_A |A\in \mathcal{A}\}$  be a class of classifiers with a set of trainable parameter configurations $\mathcal{A}$ and fixed string $s$. We say that $f$ is $(\gamma, \mathcal{S})$-compressible via $G_{\mathcal{A},s}$ if there exists $A \in \mathcal{A}$ such that for any $x \in \mathcal{S}$, we have for all $y$, $|f(x)[y] - g_A(x)[y]| \leq \gamma \text{.}$
\end{definition}
We now introduce our compression bound. The generalization error of the compressed models in expectation is, at most, the empirical generalization error of the original model if the original model has margin $\gamma$. Using standard concentration inequalities, we apply this bound over all possible pruned model outcomes. The resulting generalization bound depends on both the margin and the number of parameters in the pruned model, as in the following theorem.
% We now introduce our compression bounds. Intuitively, they are\textit{uniform convergence} bounds where they bound the difference between the error of a model on a random datapoint and the mean error on the training dataset using standard concentration inequalities. The generalization bound of a learned moodel is obtained by union bounding over all possible outputs. Compression bounds consider models with fewer parameters where the number of possible models is smaller. If there are fewer total possible models, we have to union bound over a fewer number of models which improves the tightness of the concentration inequality.  Therefore, the two parts of the bound will be upper bounding the empirical error of an output of the compression algorithm and bounding the number of parameters in that output. This is formalized in the following theorem.
\begin{theorem}
\label{thm:aroraoriginal}
\citep{Arora2018} Suppose $G_{\mathcal{A}, s} = \{ g_{A, s} | A \in \mathcal{A}\}$ where $A$ is a set of $q$ parameters each of which can have at most $r$ discrete values and $s$ is a  helper string. Let $\mathcal{S}$ be a training set with $n$ samples. If the trained classifier $f$ is $(\gamma, \mathcal{S})$-compressible  via $G_{\mathcal{A}, s}$, with helper string $s$, then there exists $A \in \mathcal{A}$ with high probability over the training set , $$R_0(g_A) \leq \hat{R}_\gamma(f) + \mathcal{O}\left(\sqrt{\frac{q \log r}{n}}\right)\text{.}$$
\end{theorem}
It is to be noted that the above theorem provides a generalization bound for the compressed classifier $g_A$, not for the trained classifier $f$. Therefore, the two parts of forming tighter generalization bounds for a given compression algorithm involve bounding the error introduced by the compression, the $\gamma$ in $\hat{R}_\gamma(f)$, and the number of parameters $q$ after compression. We demonstrate that we can achieve both with traditional MBP.
% We furthermore present a slightly editted version that may be useful for later. 

% \begin{restatable}{theorem}{editedarora}
% \label{thm:editedarora}
%     If there are $J$ different parameterizations, the generalization error of a compression $g_a$ is with probability at least $1-\delta$, $$R_0(g_A) \leq \hat{R}_{\gamma}(f) + \sqrt{\frac{\ln\left(\sqrt{\frac{J}{\delta}}\right)}{m}}\text{.}$$
% \end{restatable}
\subsection{Preliminary Assumptions}
\label{sec:preliminary}
Analyzing the effects of pruning is difficult without first understanding from which distribution the weights of a trained model lie. This is a complex and open question in general. However, \citet{han2015learning} made the empirical observation that weights often lie in a zero-mean Gaussian distribution, such as in Figure 7 in \cite{han2015learning}. We will thus assume this to be true, that the distribution of weights follows a normal distribution with $0$ mean and variance $\Psi$. Here, we state the main preliminary assumptions that we will use later.
\begin{assumption}
\label{ass:folded}
For any $l \in [L], i,j \in [d_1^l] \times [d_2^l]$, $\mathbf{A}_{i,j}^l \sim \mathcal{N}(0, \Psi)$.
\end{assumption}
 This assumption states that each atom within a matrix of the learned model obeys roughly a Gaussian distribution centered at $0$ with variance $\Psi$. While a strong assumption, this is common. In fact, \citet{Qian2021} assumes that the weights follow a uniform distribution to analyze the weights of pruned models. We assume a Gaussian distribution since this is more reasonable than the uniform distribution assumption. We can now present the MBP algorithm we will analyze throughout this paper. 

% We also will use similar noise-stability assumptions from \cite{CompressionBounds}. We simplify and combine several of their assumptions, but assume the same that the matrices in the trained model are resistant to noise. 
% \begin{assumption} Her we introduce our simple noise-stability based assumptions.
%     \begin{enumerate}
%         \item \textbf{Smoothness Between Layers} We assume that for any noise $\eta$, on layer $l$, $\mu_l$ is the smallest quantity satisfying $$\|\mathbf{A}^l(x^{l-1} + \eta) \|\leq \mu_l \|x^{l-1}\|\|\eta\|\text{.}$$
%         \item \textbf{Activation Smoothness} We assume that there exists a $c_l$ that is the smallest quantity satisfying $$\|\mathbf{A}^lx^{l-1}\| \leq c_l\|\mathbf{A}^l\sigma_l(x^{l-1})\| = c_l\|x^l\| \text{.}$$
%     \end{enumerate}    
% \end{assumption}

%   \begin{restatable}{lemma}{boundmin}
%  \label{lem:boundmin}
% The minimum term $\mathbf{A}_{\text{min}}^l$ obeys the principle $\mathbf{A}_{\text{min}}^l \geq \eta$ with probability at least $\left(1 - \operatorname{erf}\left(\frac{\eta}{\sqrt{2\Psi}}\right)\right)^{d_{1, l}d_2^l}$.
%  \end{restatable}
 \section{Magnitude-Based Pruning Algorithms}
 \label{sec:newcompression}

While many versions of MBP algorithms exist, they are all based on the general framework of removing weights of small magnitude to reduce the number of parameters while ensuring the pruned model does not differ vastly from the original. We wish to choose a pruning algorithm based on this framework often used by practitioners while at the same time being easy to analyze. We develop our algorithm to mimic the random MBP seen often in works like \citet{han2015learning, Qian2021}. While the term inside the Bernoulli random variable used as an indicator for pruning is slightly different as compared to previous literature, this is a small change that allows us to move away from the uniform distribution assumption from \citet{Qian2021} to a more favorable Gaussian assumption. We formally present our algorithm in \Cref{alg:compscheme2} below. 
\RestyleAlgo{ruled}
\begin{wrapfigure}{r}{.5\textwidth}
\begin{minipage}[t]{.5\textwidth}
\begin{algorithm}[H]
 \caption{ MBP}
 \label{alg:compscheme2}
 \SetAlgoLined
  \KwData{$\{\mathbf{A}^1, \dots, \mathbf{A}^L\},d$}
  \KwResult{$\{\mathbf{\hat{A}}^1, \dots, \mathbf{\hat{A}}^L\}$}
  \For{$l \in [L]$}{
    \For{$i,j \in [d_1^l] \times [d_2^l] \text{ and } i \neq j$}{
        $X \coloneqq \text{Bernoulli}\left(\operatorname{exp}\left(\frac{-\left[\mathbf{A}_{i, j}^l\right]^2}{d\Psi}\right)\right)$\\
        $\hat{\mathbf{A}}_{i,j}^l \coloneqq 0 \text{ if } X = 1 \text{ else }  \mathbf{A}_{i, j}^l$
    }
  }
\end{algorithm}
\vspace{-25pt}
\end{minipage}
\end{wrapfigure}
\begin{restatable}{remark}{diagonal}
    We do not prune the diagonal elements in \Cref{alg:compscheme2}. While not standard, this enables the use of Matrix Sketching later on for better generalization bounds. However, in \citet{dasarathy2013}, they note the necessity for the diagonal elements being nonzero is for ease of presentation of the proof, and Matrix Sketching should still be possible with pruning the diagonal elements.
\end{restatable}
The atom's probability of being compressed is relatively small for larger atoms. The probability of getting compressed is larger for smaller atoms closer to $0$. Here, $d$ is a hyperparameter helpful for adjusting the strength of our compression. Using this compression algorithm, the pruned model will likely maintain the connections between the larger atoms while removing many smaller parameters. To use the generalization bounds from \Cref{sec:revarora}, we need to show that \Cref{alg:compscheme2} creates a pruned model $\hat{\mathbf{M}}$ that produces outputs similar to the original model $\mathbf{M}$. We also need to show that the number of nonzero parameters in the pruned models is small. We prove this in the sections below.
\subsection{Error Proof}
We begin by bounding the difference between the outputs of corresponding layers in the pruned and original models to prove that the expected difference between the pruned and original models is small. The normality assumption from \Cref{ass:folded} makes this much more tractable to compute. Indeed, each atom of the difference matrix $\Delta^l = \hat{\mathbf{A}^l} - \mathbf{A}^l$ is an independent and identical random variable. Bounding the $\ell_2$ norm of such a matrix relies only on the rich literature studying the norms of random matrices. In fact, from \citet{Latala2005}, we only need a bounded second and fourth moment of the distribution of each atom. To utilize this bound, we only need to demonstrate that the difference matrix $\Delta^l$ and the pruned model obtained using our compression scheme \Cref{alg:compscheme2} have atoms whose moments are bounded and have zero-mean. Given the compression algorithm chosen and the distribution of weights after training using \Cref{ass:folded}, the $\Delta^l$ matrix does satisfy such properties. We demonstrate them in the following lemma. 
\begin{restatable}{lemma}{meaneletwo}
    \label{lem:meaneletwo}
    The Expected Value of any entry $\Delta_{ij}^l$ of the matrix $\Delta^l = \hat{\mathbf{A}^l} - \mathbf{A}^l$ is given by $\mathbb{E}(\Delta_{ij}^l) = 0$ for any $i,j \in [d_1^l] \times [d_2^l], l \in [L]$. Thus, $\mathbb{E}(\Delta^l) = \mathbf{0}$ is a matrix full of $0$'s. Furthermore, we have that $\mathbb{E}((\Delta_{ij}^l)^2) = \frac{d^{\frac{3}{2}}\Psi}{(d+2)^{\frac{3}{2}}}$. Moreover, the fourth moment of any entry  $(\Delta_{ij}^l)^4$ of $\Delta^l$ is given by $\mathbb{E}((\Delta_{ij}^l)^4) = \frac{3d^{\frac{5}{2}}\Psi^2}{(d+2)^{\frac{5}{2}}}$. 
\end{restatable}

Given these properties of the $\Delta^l$ matrix, we can use simple concentration inequalities to bound the error accumulated at any layer between the pruned and original models. If we simulate some sample input $x$ going through our pruned model, we can show that the error accumulated through the entire model is bounded via induction. We formally present such a lemma here. 
\begin{restatable}{lemma}{errortwo}
\label{lem:errortwo}
For any given layer $l \in [L]$, we have with probability at least $1 - \frac{1}{\epsilon_l}$ $$\norm{(\mathbf{\hat{A}}^l - \mathbf{A}^l)}_2 \leq \epsilon_l \Gamma_l \text{\quad where \quad} \Gamma_l = C\left[ \left(\sqrt{\frac{d^{\frac{3}{2}}\Psi}{(d + 2)^{\frac{3}{2}}}}\right)\left(\sqrt{d_1^l} + \sqrt{d_2^l}\right) + \left(\frac{3d_1^ld_2^ld^{\frac{5}{2}}\Psi^2}{(d+2)^{\frac{5}{2}}}\right)^\frac{1}{4}\right]\text{.}$$ Here, $\hat{\mathbf{A}}^l$ is generated by \Cref{alg:compscheme2} and $C$ is a universal positive constant.
\end{restatable}
 We now have a formal understanding of how pruning a given layer in the original model affects the outcome of that layer. We can now present our error bound for our entire sparse network. 
 \begin{restatable}{lemma}{finalerrorbound}
    \label{lem:finalerrorbound}
    The difference between outputs of the pruned model and the original model on any input $x$ is bounded by, with probability at least  $1 - \sum_l^L \epsilon_l^{-1}$,
    $$\|\hat{x}^L - x^L\| \leq  ed_1^0 \left(  \prod_{l=1}^L L_l\|\mathbf{A}^l\|_2 \right) \sum_{l=1}^{L}\frac{\epsilon_l \Gamma_l}{\|\mathbf{A}^l\|_2}\text{.}$$
\end{restatable}
This bound states that the error accumulated by the pruned model obtained using \Cref{alg:compscheme2} depends only on the dimension of the model, the Lipschitz constant of the activation functions, the variance of our entries, and the error of the compression. Such a bound is intuitive as the error is accumulated iteratively throughout the layers. We provide a brief proof sketch here.
\begin{hproof}
    By the Perturbation Bound from \citet{Neyshabur2017}, we can bound how much error accumulates through the model using the norm of the difference matrix from \Cref{lem:errortwo}. 
\end{hproof}

We can form tighter bounds by considering what the expected maximum of $(\hat{\mathbf{A}}^l - \mathbf{A}^l)x$ is with high probability. If $d_2^l < d_1^l$, we observe that the matrix $\hat{\mathbf{A}}^l - \mathbf{A}^l$ has at most $d_2^l$ nonzero singular values. For more details, please see \Cref{sec:augmentedperturb}.

However, more than this error bound is needed to prove strong generalization bounds. We require the number of possible models after training and compression to be finite to use compression bounds. Therefore, we need to apply discretization to our compressed model to ensure that the number of models is finite. This, however, is relatively simple given the theoretical background already provided.

\subsection{Discretization}
 We now show that the prediction error between a discretization of the pruned model and the original model is also bounded. Our discretization method is simply rounding each value in layer $l$ to the nearest multiple of $\rho_l$. We will call the discretized pruned model $\tilde{\mathbf{M}}$ where the $l$th layer will be denoted as $\tilde{\mathbf{A}}^l$. We provide the following lemma bounding the norm of the difference of the layers between the pruned and the discretized model. Using this intuition, we can prove that the error induced by the discretization is small.
\begin{restatable}{lemma}{discreteerror}
    \label{lem:discreteerror}
    The norm of the difference between the pruned layer and the discretized layer is upper-bounded as 
    $\|\tilde{\mathbf{A}}^l - \hat{\mathbf{A}}^l\|_2 \leq \rho_l J_l$ where $J_l$ is the number of nonzero parameters in $\hat{\mathbf{A}}^l$ ($J_l$ is used for brevity here and will be analyzed later). \label{lem:errorafterdiscretization}
    With probability at least $1- \sum_{l=1}^L \epsilon_l^{-1}$, given that the parameter $\rho_l$ for each layer is chosen such that $\rho_l \leq \frac{\frac{1}{L}\|\mathbf{A}^l\|_2 - \epsilon_l \Gamma_l}{J_l}$, we have that the error induced by both discretization and the pruning is bounded by
    $$\|x^L - \tilde{x}^L\|_2 \leq ed_1^0 \left(  \prod_{l=1}^L L_l\|\mathbf{A}^l\|_2 \right) \sum_{l=1}^{L}\frac{\epsilon_l \Gamma_l +  \rho_l J_l}{\|\mathbf{A}^l\|_2}\text{.}$$ 
    
\end{restatable}

%  Such a proof is similar to the proof of \Cref{lem:finalerrorbound}. 
% \begin{restatable}{lemma}{errorafterdiscretization}
%     \label{lem:errorafterdiscretization}
%     Given that parameter $\rho_l$ for each layer is chosen such that $\rho_l \leq \frac{\frac{1}{L}\|\mathbf{A}^l\|_2 - \epsilon_l \Gamma_l}{J_l}$, we have that the error induced by both discretization and the pruning is bounded by
%     $$\|x^L - \tilde{x}^L\|_2 \leq ed_1^0 \left(  \prod_{l=1}^L L_l\kappa_l\|\mathbf{A}^l\|_2 \right) \sum_{l=1}^{L}\frac{\epsilon_l \Gamma_l +  \rho_l J_l}{\|\mathbf{A}^l\|_2}\text{.}$$ This happens with probability $1- \sum_{l=1}^L \epsilon_l^{-1}$. 
    
% \end{restatable}

% Using such a discretization metric, we have that the number of values each parameter can take for layer $l$ is $r = \mathcal{O}\left(\frac{1}{\rho_l}\right) = \mathcal{O}\left(\frac{J_l}{\frac{1}{L}\mathbf{A}^l - \epsilon_l \Gamma_l}\right)$.
Now, we have a sufficient error bound on our MBP algorithm. Thus, as the next step, we focus on bounding the number of parameters our compressed model will have. To do this, we introduce our next significant theoretical tool: \textit{Matrix Sketching}. 
\section{Sketching Sparse Matrices}
\label{sec:sketching}
As seen in \Cref{thm:aroraoriginal}, the generalization error depends strongly on the number of parameters. We try to count the number of possible parameterizations of the pruned model $\hat{M}$ achievable by combining the learning algorithm and the compression algorithm. In the appendix, we discuss one naive approach by counting the number of possible sparse matrices generated by the combination of a learning algorithm and \Cref{alg:compscheme2}. This achieves a less-than-desirable generalization bound, forming motivation for Matrix Sketching. We now introduce the preliminaries and motivate the need for matrix sketching.

\subsection{Preliminary on Matrix Sketching}
\label{sec:premaske}
Here we introduce the preliminary concepts of matrix sketching. Namely, we can represent a sparse matrix $X \in \mathbb{R}^{p_1 \times p_2}$ as $Y \in \mathbb{R}^{m \times m}$ where $p_1, p_2 \geq m$. The idea of matrix sketching is to create an embedding for this matrix as $Y = AXB^{\top}$. Here,  the matrices $A \in \{0,1\}^{m \times p_1}$, $B \in \{0,1\}^{m \times p_2}$ are chosen before the sketching is to be done. To recover the original matrix, we solve the minimization problem 
\begin{align}
\underset{{\tilde X} \in \mathbb{R}^{p_1 \times p_2}}{\min} & \|{\tilde X}\|_1 \text{ s.t. } Y = A{\tilde X}B^\top. \label{eq:prob1}
\end{align}
If the problem from \Cref{eq:prob1} enjoys a unique minimum and that unique minimum is the true $X$, we can say that this sketching scheme is lossless. In such a case, all the information in $X$ is encoded in $Y$. Given such a mapping, we can use this one-to-one mapping between $Y$ and $X$ to count the number of parametrizations of $X$ using $Y$, which is of a smaller dimension. We use properties from this literature to help develop and prove the improved generalization error bounds.

We claim with matrix sketching that we can represent the space of large sparse matrices of dimension $p$ with the set of small dense matrices of dimension $\sqrt{jp}\log{p}$ where $j$ is the maximum number of nonzero elements in any row or column. Counting the number of parameters in small dense matrices is more efficient in terms of parameters than counting the number of large sparse matrices, thus providing a way of evasion of the combinatorial explosion. We formalize this in the following section. 
\subsection{Sparse Case}
To apply the matrix sketching literature to our sparse matrices, we need to prove several properties of the matrices obtained using our compression scheme \Cref{alg:compscheme2}. We introduce one such structure called \emph{$j_r, j_c$-distributed-sparsity}, which ensures sketching can be applied to matrices. Intuitively, such a property ensures that any row or column of our sparse matrix does not contain too many nonzero elements. We formally define such intuition here.
\begin{definition}
    A matrix is $j_r, j_c$-distributed sparse if at most $j_r$ elements in any column are nonzero, $j_c$ elements in any row are nonzero, and the diagonal elements are all nonzero. 
\end{definition}
The other main knowledge is how to form $A,B$ for sparse matrix sketching. If the reader is interested, we discuss how to form $A$ and $B$ alongside some intuition behind matrix sketching in \Cref{sec:chooseab}.

\subsection{Bounds for Sparse Matrix Sketching}
From \citet{dasarathy2013}, it can be proved that sketching the set of $j_r, j_c$-distributed sparse matrices requires only small $m$. Given a choice of $m$ and probability term $ \delta$, one can show that the solution to \Cref{eq:prob1} matches the uncompressed value with high probability. This is mainly shown by first demonstrating that a solution exists and that the best solution to $A^{-1}YB^{-1} = {\tilde X}$ is the only solution that minimizes the $\ell_1$ norm with high probability. 
\begin{theorem}
\label{thm:dasarthy}
    (From Theorem 1 of \cite{dasarathy2013}) Let $p=\max(d_1^l, d_2^l)$. Suppose that $A \in \{0,1\}^{m \times d_1^l}$, $B \in \{0,1\}^{m \times  d_2^l}$ are drawn independently and uniformly from the $\delta$-random bipartite ensemble. Then, as long as $m = \mathcal{O}(\sqrt{\max(j_cd_1^l, j_rd_2^l) }\log(p))$ and $\delta = \mathcal{O}(\log(p))$, there exists a $c \geq 0$ such that for any given $j_r, j_c$-distributed sparse matrix $X$, sketches $AXB$ into $\tilde{X}$ results in a unique sketch for each $X$.  This statement holds with probability $1-p^{-c}$.
\end{theorem}
\begin{restatable}{remark}{remcval}
    The theorem statement for \Cref{thm:dasarthy} states that $c \geq 0$. However, in the proof, they demonstrate the stronger claim that $c \geq 2$. Therefore, the probability that \Cref{thm:dasarthy} holds is at least $1 - p^{-2}$. As $p$ grows, the probability that this theorem holds approaches $1$. 
\end{restatable}
\subsection{Generalization Error from Sketching}
To use the above theoretical tools of matrix sketching, we must show that outputs from our compression algorithm \Cref{alg:compscheme2} satisfy the definitions of $j_r, j_c$-distributed-sparsity. Such a claim is intuitive and similar properties have been shown for random matrices following different distributions. Given that our trained matrices satisfy the Gaussian distribution, one row or column is unlikely to contain many nonzero elements. Here, we prove in the following lemma that the pruned model using \Cref{alg:compscheme2} satisfies the condition of distributed sparsity using \Cref{ass:folded}.
\begin{restatable}{lemma}{algdisspar}
\label{lem:algdisspar}
With probability at least $1 - \frac{1}{\lambda_l} -(d_1^l)^{-\frac{1}{3}} - (d_2^l)^{-\frac{1}{3}}$, we have that the outputs from \Cref{alg:compscheme2} are $j_r, j_c$-sparsely distributed where $\max(j_r, j_c) \leq 3 \lambda_l \max(d_1^l, d_2^l) \chi$ and $\lambda_l \in \mathbb{R}$. Here, $\chi = \frac{\sqrt{d+2}  - \sqrt{d}}{\sqrt{d+2}}$.
\end{restatable}
% Given that we have \Cref{lem:algdisspar}, we can now efficiently count the number of parameters in compressed models by counting the number of parameters in smaller dense matrices. We provide this formally in the following corollary.
% \begin{restatable}{corollary}{numparamdensesmall}
%     Given that the sketching matrices is lossless, there are $r^{m^2}$ possible small matrices and equivalently $r^{m^2}$ sparse matrices.
% \end{restatable}
Given the above quantification of the space of sparse matrices and the bound of the error of our model, we can apply the compression bound from \cite{Arora2018}. Such a compression bound intuitively depends mainly on these two values.  
\begin{restatable}{theorem}{generrorsketch}
\label{thm:generrorsketch}
For every matrix $\hat{\mathbf{A}}^l$, define $j_l$ to be the $\max(j_r, j_c)$ where $j_r$ and $j_c$ are the distribution-sparsity coefficients for $\hat{\mathbf{A}}^l$. Moreover, for every matrix $\hat{\mathbf{A}}^l$, define $p_l = \max(d_1^l, d_2^l)$. Then we have that
$$R_0(g_A) \leq \hat{R}_\gamma(f) + \mathcal{O}\left(\sqrt{\frac{\sum_l 3 \lambda_l \chi d_2^ld_1^l\log^2(p_l) \log(\frac{1}{\rho_l})}{n}}\right)\text{.}$$
This holds when $d$ is chosen such that $\gamma \geq ed_1^0 \left(  \prod_{l=1}^d L_l\|\mathbf{A}^l\|_2 \right) \sum_{l=1}^{L}\frac{\epsilon_l \Gamma_l +  \rho_l  J_l}{\|\mathbf{A}^l\|_2}$ where $J_l \leq \mathcal{O}\left(\chi d_2^ld_1^l\right)$.
This claim holds with probability at least $1 -  \left[\sum_{l=1}^L \lambda_l^{-1} + \epsilon_l^{-1} + p_l^{-c}\right]$. 
\end{restatable}
Here, $\chi$ depends on our hyperparameter $d$. Indeed, this generalization bound vastly improves on a trivial application of compression bounds as in \Cref{lem:badgeneralization}. Quite notably, this removes the combinatorial nature of the naive bound in \Cref{lem:badgeneralization} for a small price of $\sqrt{\frac{\log^2(p_l)}{n}}$.  
\section{Generalization of Lottery Tickets}
\label{sec:subweightbounds}
Such a generalization error proof framework for pruning applies to more than just Magnitude-based pruning. An exciting pruning approach is simply creating a very large model $G$ such that some smaller target model $\mathbf{M}$ is hidden inside $G$ and can be found by pruning. This lottery ticket formulation for pruning has seen many empirical benefits. Formally, we will call our lottery ticket within $G$ a weight-subnetwork $\widetilde{G}$ of $G$. This $\widetilde{G}$ is a pruned version of the original $G$. In fact, \citet{Malach2020} shows that for a sufficiently large $G$, there exists with high probability a pruning $\tilde{G}$ such that $\tilde{G}$ and the target function $\mathbf{M}$ differ by at most $\epsilon$. Moreover, this $\tilde{G}$ will have approximately the same number of nonzero parameters as the original target network $\mathbf{M}$. This is formally presented in \Cref{thm:wtsubnetwork}.

\begin{restatable}{theorem}{wtsubnetwork}
\label{thm:wtsubnetwork} Fix some $\epsilon, \delta \in(0,1)$. Let $\mathbf{M}$ be some target network of depth $L$ such that for every $i \in[L]$ we have $\|\mathbf{A}^i\|_2 \leq 1,\left\|\mathbf{A}^i\right\|_{\max } \leq \frac{1}{\sqrt{d_{1, i}}}$. Furthermore, let $n_\mathbf{M}$ be the maximum hidden dimension of $\mathbf{M}$. Let $G$ be a network where each of the hidden dimensions is upper bounded by $\operatorname{poly}\left(d_{1,0}, n_\mathbf{M}, L, \frac{1}{\epsilon}, \log \frac{1}{\delta}\right) \coloneqq D_G$ and depth $2 L$, where we initialize $\mathbf{A}^i$ from the uniform distribution $U([-1,1])$. Then, w.p at least $1-\delta$ there exists a weight-subnetwork $\widetilde{G}$ of $G$ such that:
$$
\sup _{x \in \mathcal{S}}|\widetilde{G}(x)-\mathbf{M}(x)| \leq \epsilon.
$$
Furthermore, the number of active (nonzero) weights in $\widetilde{G}$ is $O\left(d_{0,1} D_G+D_G^2 L\right)$.
\end{restatable}

\begin{restatable}{theorem}{impgen}
    \label{thm:impgen}
    With probability at least $1-\delta - LD_G^{-c}$ have the generalization error of 
    $$R_0(\tilde{G}) \leq \hat{R}_{\epsilon + \epsilon_{\rho}}(\tilde{G}) + \mathcal{O}\left(\sqrt{\frac{\left[n_\mathbf{M}d_{0,1}\log(D_G)^2 + Ln_\mathbf{M}^2\log(D_G)^2\right]\log\left(\frac{1}{\rho}\right)}{n} }\right)\text{.}$$ Here, $\epsilon_{\rho}$ is the small error introduced by discretization. 
\end{restatable}
Here, we have a generalization bound for our pruned model. One interesting thing to note is that this bound is only a small factor of $\log(D_G)$ worse than if we had applied a compression bound to a model of the size of the target function $\mathbf{M}$. To our knowledge, this is one of the first generalization analyses for lottery tickets. 

\section{Empirical Analysis}
\label{sec:experiments}

% 1. What do we want to verify?
% a. Our bounds are tighter than what currnelty exist for norm based bounds
% 2. What baselines?
% a. Neyshabur 2015, Bartlett, neyshabur 2017
% 3. What are parameters of experiments
% a. Model Details
% b. Training Details
% c. Dataset details
% 4. What extra dependence tests do we want to show
% a. Dependence on Epoch
% b. Dependence on dimension

\paragraph{Code} We have provided our code \href{https://github.com/EtashGuha/compression_bounds_arxiv}{\textbf{here}} for reproducibility. This was based on a fork of \citet{labonte23}.

We study the generalization bound obtained using our pruning Algorithm \ref{alg:compscheme2}  with some standard well-known norm-based generalization bounds of \citet{neyshabur2015norm},\citet{Bartlett2017}, and \citet{Neyshabur2017} used as a baseline.  Our experiments compare the generalization error obtained by these bounds, the generalization bound of our algorithm (as provided in \ref{thm:generrorsketch}), and the true generalization error of the compressed, trained model. We also provide an experiment demonstrating how our generalization bound scales when increasing the hidden dimension of our models.

Our models are Multi-Layer Perceptron Models (MLPs) with ReLU activation layers with $5$ layers. We train our algorithm with a learning rate of 
$0.02$ with the Adam optimizer \citep{kingma2014adam} for $300$ epochs. We conduct our experiments on two different image classification datasets: MNIST \citep{mnist} and CIFAR10 \citep{cifar10}. We use an MLP with a hidden dimension of $784$ to compare our bounds to other generalization bounds. For our experiments on scaling with model size, we test on hidden dimensions $500, 1000, 1500, 2000, \text{ and } 2500$ where the depth is kept constant.

\begin{figure}
     \centering
     \begin{subfigure}[b]{0.32\textwidth}
         \centering
         \includegraphics[width=\textwidth]{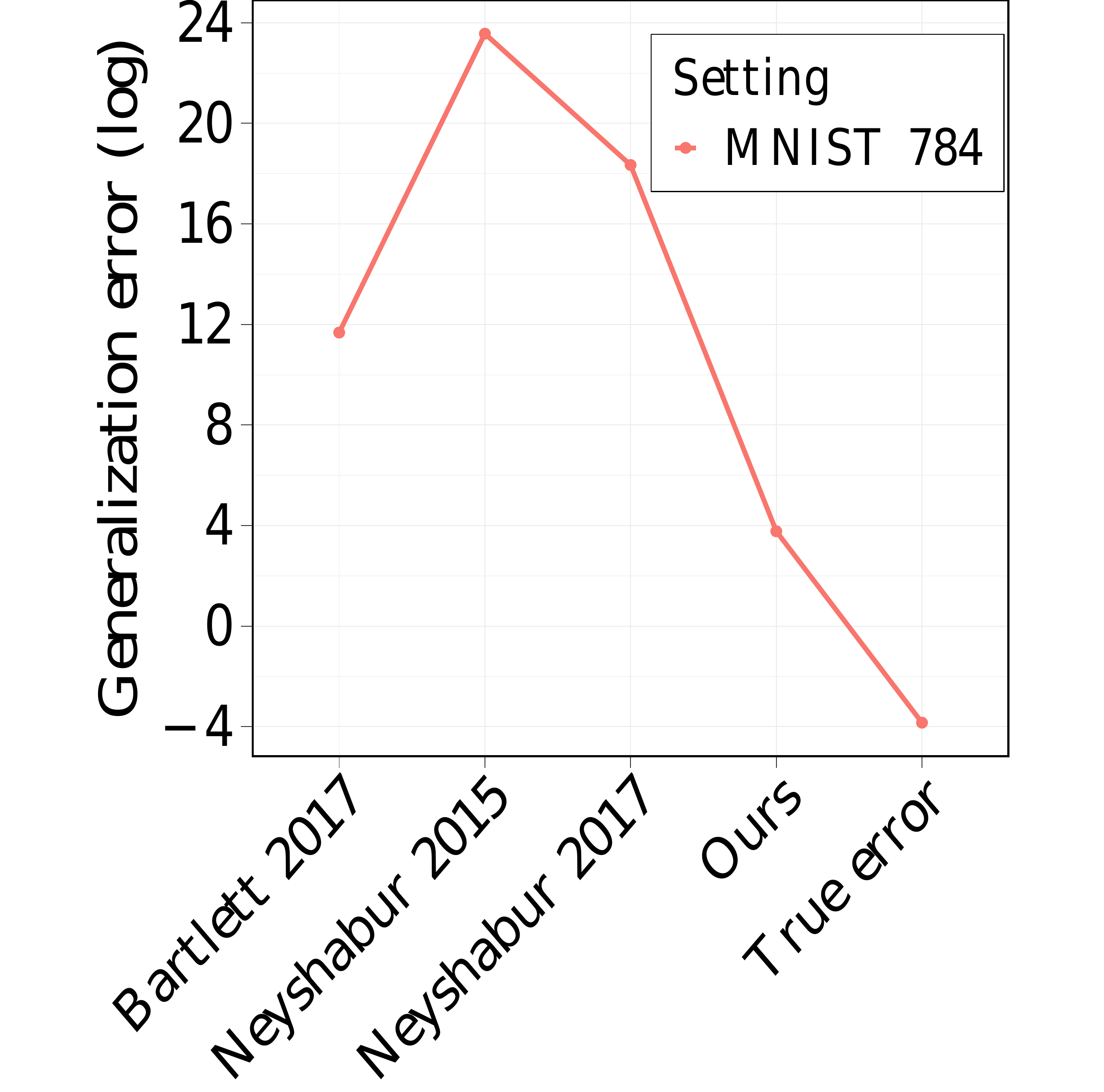}
         \caption{Comparing bounds on MNIST.}
         \label{fig:mnist}
     \end{subfigure}
     \hfill
     \begin{subfigure}[b]{0.32\textwidth}
         \centering
         \includegraphics[width=\textwidth]{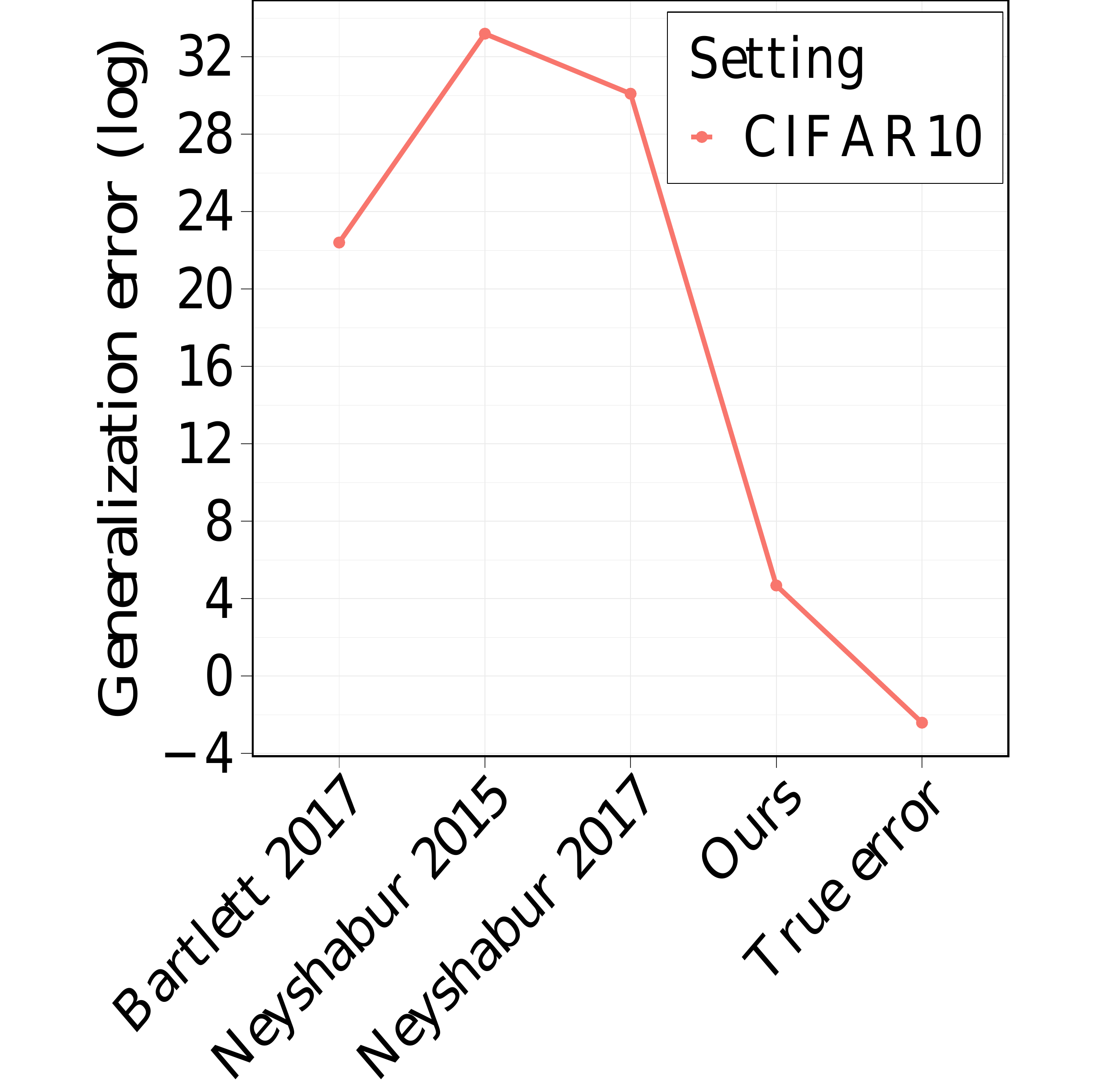}
         \caption{Comparing bounds on CIFAR10.}
         \label{fig:cifar}
     \end{subfigure}
    \hfill
     \begin{subfigure}[b]{0.32\textwidth}
         \centering
         \includegraphics[width=\textwidth]{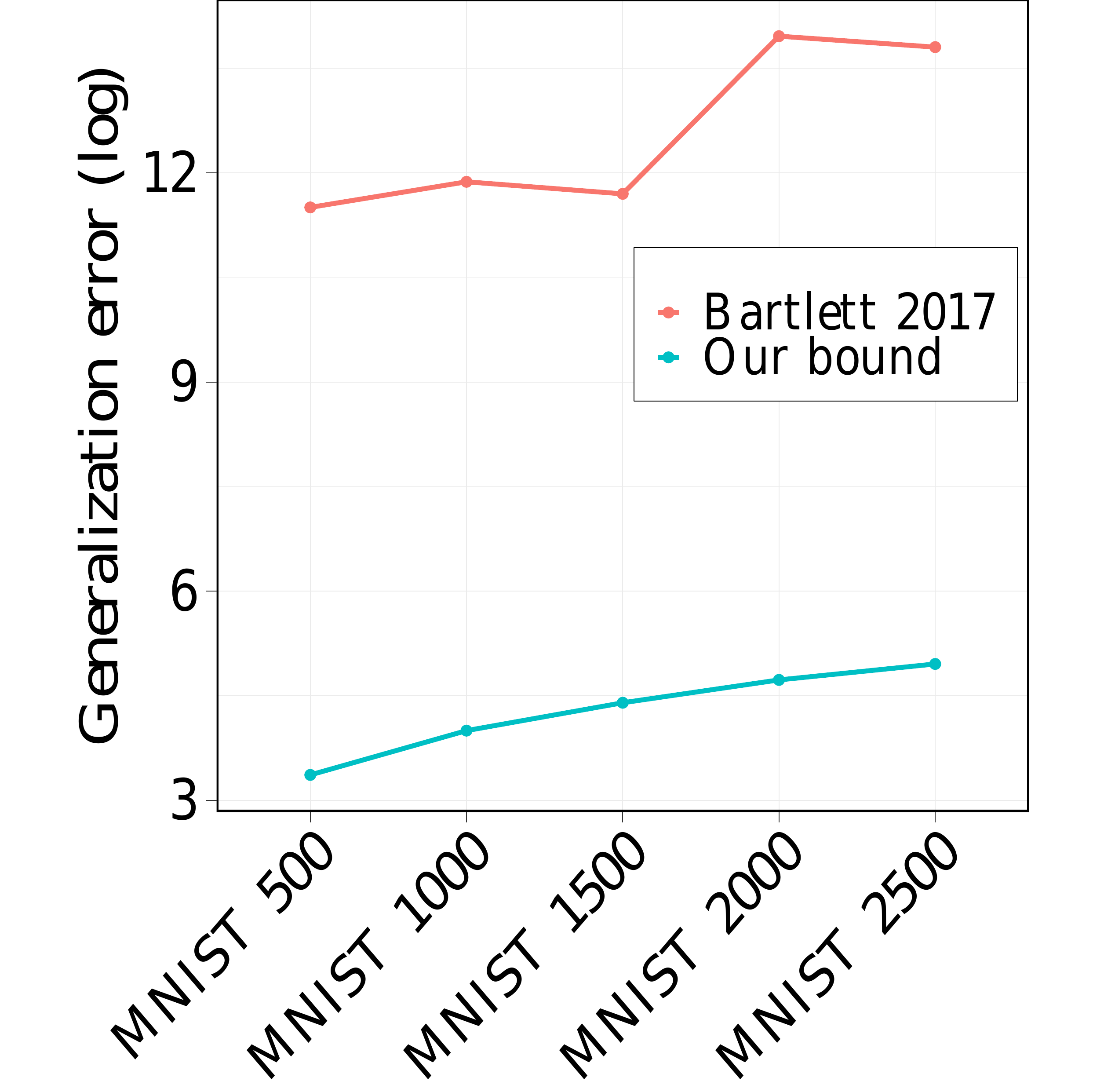}
         \caption{Dependence on model size.}
         \label{fig:modelsize}
     \end{subfigure}
        \caption{Comparison of the generalization bounds on logarithmic scale w.r.t. (a) MNIST, and (b) CIFAR10 datasets. In (c), we see how our bounds depend on the size of the model.}
        \label{fig:experiments}
\end{figure}

\paragraph{Results}Our bounds are several orders of magnitude better than the baseline state-of-the-art generalization bounds, as can be inferred from \Cref{fig:mnist} and \Cref{fig:cifar} above. In both experiments, the closest bound to ours is that of \citet{Bartlett2017}, which is still at least $10^3$ and $10^7$ times greater than our bounds on the MNIST and the CIFAR10 datasets respectively. Moreover, our generalization bound is consistently better than \citet{Bartlett2017} as the hidden dimension grows \Cref{fig:modelsize}. This demonstrates that across several datasets, our bounds are tighter than traditional norm-based generalization bounds and scale better with hidden dimensions. We provide some additional experiments in the appendix as well. The results, although remarkable, are not surprising mainly due to the use of our pruning algorithm, which ensures the error due to compression is low, and making use of Sparse Matrix Sketching, which significantly reduces the dimension of the pruned model, a key factor while computing generalization bounds. 

\section{Conclusion}
This paper has made progress on the problem of deriving generalization bounds of overparametrized neural networks. With our efficient pruning algorithm and Sparse Matrix Sketching, we have obtained bounds for pruned models which are significantly better than well-known norm-based generalization bounds and empirically verified the effectiveness of this approach on actual data. We hope these results will fuel further research in deep learning to understand better how and why models generalize. It would be interesting to see if matrix sketching can be used to prove the generalization for different types of pruning, such as coresets. Moreover, it could also be fruitful to see whether matrix sketching can be used alongside PAC-Bayes bounds to yield generalization bounds as well. In this regard, we have extended the general framework of this paper to derive generalization bounds for lottery tickets in Section \ref{sec:subweightbounds}, a result which, to our knowledge, is the first of its kind. Another possibility would be to explore how different generalization bounds can be formed for different data distributions from different training algorithms. 

\paragraph{Limitations} Our magnitude-based pruning algorithm does not prune the diagonal of the matrix, which is not standard. Moreover, after training, we assume that each atom belongs to an i.i.d Gaussian distribution, which may not always hold. Also, similar to the standard bounds, our generalization bounds are still vacuous, not fully explaining the generalization of the models.

%Our analysis makes an assumption on the distribution of weights after training that may vary depending on the situation. 

%Also, our bounds still do not capture the generalization error curve that is seen throughout training as seen by our experiment in \Cref{alg:compscheme1}.
\bibliographystyle{plainnat}
\bibliography{cite}
\appendix

\section{Computation Details}
Here, we provide some information about the hardware and setup used for our work. We include a copy of the code in the Supplementary Material for reproducibility. Our experiments were run with an Intel(R) Xeon(R) Silver 4116 CPU @ 2.10GHz and an NVIDIA GeForce RTX 2080 Ti. Moreover, our experiments are run using Python version $3.10$. Also, our experiments were done using a batch size of $256$ for our experiments on MNIST and a batch size of $4$ for CIFAR10. The Deep Learning software library used was PyTorch Lightning. We have additionally included our code in the supplementary material alongside with a README on how to reproduce our experiments.

\section{Proof of Pruning and Discretization Error} 
We state that much of this error analysis is inspired by the proofs in \citet{Qian2021}. We, however, use the more reasonable Gaussian distribution assumption over the weights and have a slightly different Magnitude-based Pruning algorithm. 

\subsection{Proof of \Cref{lem:meaneletwo}}
\meaneletwo*
\begin{proof}
By the definition of our \Cref{alg:compscheme2}, the random variable $\Delta_{ij}^l$ depends on the random variable $\mathbf{A}_{i,j}^l$. 

\begin{equation}
\Delta_{ij}^l =
    \begin{cases}
        \mathbf{A}_{i,j}^l & \text{w.p. } \operatorname{exp}\left(\frac{-\left[\mathbf{A}_{i, j}^l\right]^2}{d\Psi}\right)\nonumber\\
        0 & \text{w.p. } 1 - \operatorname{exp}\left(\frac{-\left[\mathbf{A}_{i, j}^l\right]^2}{d\Psi}\right)
    \end{cases}.
\end{equation}

To calculate $\mathbb{E}\left(\Delta_{ij}^l\right)$, we will use the Law of Total Expectation. That is, $\mathbb{E}\left(\Delta_{ij}^l\right)= \mathbb{E}(\mathbb{E}(\Delta_{ij}^l | \mathbf{A}_{i,j}^l))$. We have that, $\mathbb{E}(\Delta_{ij}^l | \mathbf{A}_{i,j}^l) = \mathbf{A}_{i,j}^l\operatorname{exp}\left(\frac{-\left[\mathbf{A}_{i, j}^l\right]^2}{d\Psi}\right)$. Therefore, to calculate $\mathbb{E}(\mathbb{E}(\Delta_{ij}^l | \mathbf{A}_{i,j}^l))$, we use the definition of expectation for continuous variables as 
\begin{align}
\mathbb{E}(\Delta_{ij}^l) &= \mathbb{E}(\mathbb{E}(\Delta_{ij}^l | \mathbf{A}_{i,j}^l))\nonumber\\
&= \int_{-\infty}^{\infty} \mathbf{A}_{i,j}^l\operatorname{exp}\left(\frac{-\left[\mathbf{A}_{i, j}^l\right]^2}{d\Psi}\right)\frac{1}{\sqrt{2\pi\Psi}}\operatorname{exp}\left(-\frac{1}{2}\left(\frac{\mathbf{A}_{i,j}^l}{\sqrt{\Psi}}\right)^2\right)d\mathbf{A}_{i,j}^l\nonumber\\
&= 0\nonumber.
\end{align}
We now focus on the squared component of our lemma. Similarly, we have 
\begin{align}
\mathbb{E}((\Delta_{ij}^l)^2) &= \mathbb{E}(\mathbb{E}((\Delta_{ij}^l)^2 | \mathbf{A}_{i,j}^l))\nonumber\\
&= \int_{-\infty}^{\infty} (\mathbf{A}_{i,j}^l)^2\operatorname{exp}\left(\frac{-\left[\mathbf{A}_{i, j}^l\right]^2}{d\Psi}\right)\frac{1}{\sqrt{2\pi\Psi}}\operatorname{exp}\left(-\frac{1}{2}\left(\frac{\mathbf{A}_{i,j}^l}{\sqrt{\Psi}}\right)^2\right)d\mathbf{A}_{i,j}^l\nonumber\\
&= \frac{d^{\frac{3}{2}}\Psi}{(d+2)^{\frac{3}{2}}}\nonumber.
\end{align}
We similarly compute the fourth moment as 
\begin{align}
\mathbb{E}((\Delta_{ij}^l)^4) &= \mathbb{E}(\mathbb{E}((\Delta_{ij}^l)^4| \mathbf{A}_{i,j}^l))\nonumber\\
&= \int_{-\infty}^{\infty} (\mathbf{A}_{i,j}^l)^4\operatorname{exp}\left(\frac{-\left[\mathbf{A}_{i, j}^l\right]^2}{d\Psi}\right)\frac{1}{\sqrt{2\pi\Psi}}\operatorname{exp}\left(-\frac{1}{2}\left(\frac{\mathbf{A}_{i,j}^l}{\sqrt{\Psi}}\right)^2\right)d\mathbf{A}_{i,j}^l\nonumber\\
&= \frac{3d^{\frac{5}{2}}\Psi^2}{(d+2)^{\frac{5}{2}}} \nonumber.
\end{align}
\end{proof}

\subsection{Proof of \Cref{lem:errortwo}}
\errortwo*
We will prove that the error from the compression is bounded. To do so will first use the result that the expected norm of any zero-mean matrix can be bounded using the second and fourth moments. We restate this useful technical lemma from Theorem 2 of \citet{Latala2005}.
\begin{restatable}{lemma}{latala}
    \label{lem:latala}
    Let $A$ be a random matrix whose entries $A_{i,j}$ are independent mean zero random variables with finite fourth moment. Then $$\mathbb{E}\|A\|_2 \leq C\left[\underset{i}{\max}\left( \sum_j \mathbb{E} A_{i,j}^2\right)^{\frac{1}{2}} + \underset{j}{\max}\left( \sum_i \mathbb{E} A_{i,j}^2\right)^{\frac{1}{2}} + \left( \sum_j \mathbb{E} A_{i,j}^4\right)^{\frac{1}{4}}\right]\text{.}$$ Here, $C$ , is a universal positive constant.
\end{restatable}
We now use this lemma to bound the error due to compression using \Cref{alg:compscheme2}. 
\begin{proof}
    Let $Z$ be our mask such that $\hat{\mathbf{A}}_{i,j}^l = Z \circ \mathbf{A}_{i,j}^l$ where $\circ$ is the elementwise-matrix product. We will analyze the difference matrix $\Delta =  Z \circ \mathbf{A}_{i,j}^l -  \mathbf{A}_{i,j}^l$. Note that 
    $$\mathbb{E}((\Delta_{ij}^l)^2) = \mathbb{E}((\Delta_{ij}^l)^2|Z_{i, j} = 0)\cdot \mathbb{P}(Z_{i,j} = 0) + \mathbb{E}((\Delta_{ij}^l)^2|Z_{i, j} = 1)\cdot \mathbb{P}(Z_{i,j} = 1)\text{.}$$
    Trivially, if the mask for an atom is set to $1$, the squared error for that atom is  $0$. Therefore, we have that
    $$\mathbb{E}(\Delta_{ij}^l|Z_{i, j} = 1) \mathbb{P}(Z_{i,j} = 1) = 0\text{.}$$
    Thus, we only need to analyze the second term. We have
    \begin{align}
        \mathbb{E}((\Delta_{ij}^l)^2|Z_{i, j} = 0)\cdot \mathbb{P}(Z_{i,j} = 0) &= \mathbb{P}(Z_{i,j} = 0) \int_{-\infty}^{\infty} (\mathbf{A}_{i,j}^l)^2 \cdot \mathbb{P}(\mathbf{A}_{i,j}^l|Z_{i,j} = 0) d\mathbf{A}_{i,j}^l\nonumber\\
        &=  \int_{-\infty}^{\infty} (\mathbf{A}_{i,j}^l)^2 \cdot \mathbb{P}(Z_{i,j} = 0|\mathbf{A}_{i,j}^l)\cdot \mathbb{P}(\mathbf{A}_{i,j}^l) d\mathbf{A}_{i,j}^l\nonumber\\
        &=  \int_{-\infty}^{\infty} (\mathbf{A}_{i,j}^l)^2 \cdot \operatorname{exp}\left(\frac{-[\mathbf{A}_{i,j}^l]^2}{d\Psi}\right)\cdot \frac{1}{\sqrt{2\pi \Psi}}  \operatorname{exp}\left(-\frac{1}{2}\left(\frac{\mathbf{A}_{i,j}^l}{\sqrt{\Psi}}\right)^2\right) d\mathbf{A}_{i,j}^l \nonumber\\
        &= \frac{d^{\frac{3}{2}}\Psi}{(d + 2)^{\frac{3}{2}}}\nonumber.
    \end{align}
    We then have that  
    $$\mathbb{E}((\Delta_{ij}^l)^2) = \frac{d^{\frac{3}{2}}\Psi}{(d + 2)^{\frac{3}{2}}}\text{.}$$
    Similarly, we can do the same for the fourth moment.
    \begin{align}
        \mathbb{E}((\Delta_{ij}^l)^4|Z_{i,j} = 0) \cdot \mathbb{P}(Z_{i,j} = 0) &= \mathbb{P}(Z_{i,j} = 0) \int_{-\infty}^{\infty} (\mathbf{A}_{i,j}^l)^4 \cdot \mathbb{P}(\mathbf{A}_{i,j}^l|Z_{i,j} = 0) d\mathbf{A}_{i,j}^l\nonumber\nonumber\\
        &=  \int_{-\infty}^{\infty} (\mathbf{A}_{i,j}^l)^4 \cdot \mathbb{P}(Z_{i,j} = 0|\mathbf{A}_{i,j}^l)\cdot \mathbb{P}(\mathbf{A}_{i,j}^l) d\mathbf{A}_{i,j}^l\nonumber\\
        &=  \int_{-\infty}^{\infty} (\mathbf{A}_{i,j}^l)^4 \cdot \mathbb{P}(Z_{i,j} = 0|\mathbf{A}_{i,j}^l)\cdot \mathbb{P}(\mathbf{A}_{i,j}^l) d\mathbf{A}_{i,j}^l\nonumber\\
        &=  \int_{-\infty}^{\infty} (\mathbf{A}_{i,j}^l)^4 \cdot \operatorname{exp}\left(\frac{-[\mathbf{A}_{i,j}^l]^2}{d\Psi}\right)\cdot \frac{1}{\sqrt{2\pi \Psi}}  \operatorname{exp}\left(-\frac{1}{2}\left(\frac{\mathbf{A}_{i,j}^l}{\sqrt{\Psi}}\right)^2\right) d\mathbf{A}_{i,j}^l \nonumber\\
        &= \frac{3d^{\frac{5}{2}}\Psi^2}{(d+2)^{\frac{5}{2}}}\nonumber.
    \end{align}

Combining this with \Cref{lem:latala}, we have, 
$$\mathbb{E}\|\Delta^l\|_2 \leq C\left[ \left(\sqrt{\frac{d^{\frac{3}{2}}\Psi}{(d + 2)^{\frac{3}{2}}}}\right)\left(\sqrt{d_1^l} + \sqrt{d_2^l}\right) + \left(\frac{3d_1^ld_2^ld^{\frac{5}{2}}\Psi^2}{(d+2)^{\frac{5}{2}}}\right)^\frac{1}{4}\right]\text{.}$$
We can then apply Markov's inequality where 
$$\mathbb{P}(\|\Delta^l\|_2 \geq t) \leq \frac{\mathbb{E}\|\Delta^l\|_2}{t}\text{.}$$
We set $\Gamma_l = C\left[ \left(\sqrt{\frac{d^{\frac{3}{2}}\Psi}{(d + 2)^{\frac{3}{2}}}}\right)\left(\sqrt{d_1^l} + \sqrt{d_2^l}\right) + \left(\frac{3d_1^ld_2^ld^{\frac{5}{2}}\Psi^2}{(d+2)^{\frac{5}{2}}}\right)^\frac{1}{4}\right]$ for notational ease. If we set $t = \epsilon_l \Gamma_l$, then we have with probability at least $1 - \frac{1}{\epsilon_l}$ that,
$$\|\Delta^l\|_2 \leq \epsilon_l \Gamma_l\text{.}$$
We have proven our statement. 
\end{proof}
\subsection{Proof of True Perturbation Bound}
\begin{restatable}{lemma}{Neyshabur2017}
    \label{lem:Neyshabur2017}
    For the weights of the model $\mathbf{M}$ and any perturbation $\mathbf{U}^l$ for $l \in [h]$ where the perturbed layer $l$ is $\mathbf{U}^l + \mathbf{A}^l $, given that $\|\mathbf{U}^l\|_2 \leq \frac{1}{L} \|\mathbf{A}^l\|_2$, we have that for all input $x^0 \in \mathcal{S}$, 
    $$\|x^L - \bar{x}^L \|_2 \leq ed_1^0 \left(  \prod_{l=1}^L \kappa_lL_l\|\mathbf{A}^l\|_2 \right) \sum_{l=1}^{L}\frac{\|\mathbf{U}^l\|_2}{\|\mathbf{A}^l\|_2}\text{.}$$ Here, $\bar{x}^L$ denotes the output of the $L$th layer of the perturbed model. 
\end{restatable}
\begin{proof}
    This proof mainly follows from \citet{Neyshabur2017}. We restate it here with the differing notation for clarity and completeness. We will prove the induction hypothesis that $\|\bar{x}^l - x^l\|_2 \leq \left(1 + \frac{1}{L}\right)^l \|x^0\|_2 \left(  \prod_{i=1}^l L_i\|\mathbf{A}^l\|_2 \right) \sum_{i=1}^{l}\frac{\|\mathbf{U}^i\|_2}{\|\mathbf{A}^i\|_2}$. The base case of induction trivially holds, given we have that $\|\bar{x}^0 - x^0\|_2 = 0$ by definition. Now, we prove the induction step. Assume that the induction hypothesis holds for $l$. We will prove that it holds for $l+1$. 
    We have that
    \begin{align}
        \|x^l - \bar{x}^l\|_2 &\leq \|\left(\mathbf{A}^l + \mathbf{U}^l\right)\phi_l(\bar{x}^{l-1}) - \mathbf{A}^l\phi_l(x^{l-1})\|_2\nonumber\\
        &\leq \|\left(\mathbf{A}^l + \mathbf{U}^l\right)(\phi_l(\bar{x}^{l-1}) - \phi_l(x^{l-1})) + \mathbf{U}^l\phi_l(x^{l-1})\|_2\nonumber\\
        &\leq \left(\|\mathbf{A}^l\|_2 + \|\mathbf{U}^l\|_2\right)\|\phi_l(\bar{x}^{l-1}) - \phi_l(x^{l-1})\|_2 + \|\mathbf{U}^l\|_2\|\phi_l(x^{l-1})\|_2\label{eq:applykappa}\\
        &\leq \left(\|\mathbf{A}^l\|_2 + \|\mathbf{U}^l\|_2\right)\|\phi_l(\bar{x}^{l-1}) - \phi_l(x^{l-1})\|_2 + \|\mathbf{U}^l\|_2\|\phi_l(x^{l-1})\|_2\nonumber\\
        &\leq L_l  \left(\|\mathbf{A}^l\|_2 + \|\mathbf{U}^l\|_2\right)\|\bar{x}^{l-1} - x^{l-1}\|_2 + L_l \|\mathbf{U}^l\|_2\|x^{l-1}\|_2 \label{eq:ney2}\\
        &\leq L_l  \left(1 +\frac{1}{d}\right)\left(\|\mathbf{A}^l\|_2\right)\|\bar{x}^{l-1} - x^{l-1}\|_2 + L_l \|\mathbf{U}^l\|_2\|x^{l-1}\|_2\nonumber\\
        &\leq L_l  \left(1 +\frac{1}{d}\right)\left(\|\mathbf{A}^l\|_2\right)\left(1 + \frac{1}{L}\right)^{l-1} \|x^0\|_2 \left(  \prod_{i=1}^{l-1} L_i \|\mathbf{A}^i\|_2 \right) \sum_{i=1}^{l-1}\frac{\|\mathbf{U}^i\|_2}{\|\mathbf{A}^i\|_2} + L_l \|\mathbf{U}^l\|_2\|x^{l-1}\|_2 \label{eq:ney1}\\
        &\leq L_l  \left(1 +\frac{1}{L}\right)^{l}\left(  \prod_{i=1}^{l-1} L_i \|\mathbf{A}^i\|_2 \right) \sum_{i=1}^{l-1}\frac{\|\mathbf{U}^i\|_2}{\|\mathbf{A}^i\|_2}\|x^0\|_2  + L_l \|x^0\|_2 \|\mathbf{U}^l\|_2\prod_{i=1}^{l-1}L_{i}\|\mathbf{A}^i\|_2\nonumber\\
        &\leq L_l  \left(1 +\frac{1}{L}\right)^{l}\left(  \prod_{i=1}^{l-1} L_i  \|\mathbf{A}^i\|_2 \right) \sum_{i=1}^{l-1}\frac{\|\mathbf{U}^i\|_2}{\|\mathbf{A}^i\|_2}\|x^0\|_2 + \|x^0\|_2 \frac{\|\mathbf{U}^l\|_2}{\|\mathbf{A}^l\|_2}\prod_{i=1}^{l}L_{i}\|\mathbf{A}^i\|_2\nonumber\\
        &\leq \left(1 +\frac{1}{L}\right)^{l}\left(  \prod_{i=1}^{l} L_i \|\mathbf{A}^i\|_2 \right) \sum_{i=1}^{l}\frac{\|\mathbf{U}^i\|_2}{\|\mathbf{A}^i\|_2}\|x^0\|_2 \nonumber
    \end{align}
    Here, \Cref{eq:applykappa} results from applying \Cref{lem:notworstcaseerror}. \Cref{eq:ney2} comes from the fact that $\phi_i$ is $L_i$-Lipschitz smooth and that $\phi_i(0) = 0$. Moreover, \Cref{eq:ney1} comes from applying the induction hypothesis. Therefore, we have proven the induction hypothesis for all layers. We now only need the fact that $\left(1 + \frac{1}{L}\right)^L \leq e$, and we have our final statement. 
\end{proof}
\subsection{Proof of \Cref{lem:finalerrorbound}}
 \begin{restatable}{lemma}{finalerrorbound}
    \label{lem:finalerrorbound}
    The difference between outputs of the pruned model and the original model on any input $x$ is bounded by, with probability at least  $1 - \sum_i^L \epsilon_i^{-1}$,
    $$\|\hat{x}^L - x^L\| \leq  ed_1^0 \left(  \prod_{l=1}^L L_l\|\mathbf{A}^l\|_2 \right) \sum_{l=1}^{L}\frac{\epsilon_l \Gamma_l}{\|\mathbf{A}^l\|_2}\text{.}$$ 
\end{restatable}
\begin{proof}
  We will compare the output of the original model $x^l$ with the output of the compressed model. We need the fact that $\frac{1}{L} \|\mathbf{A}^l\|_2 \geq \epsilon_l \Gamma_l \geq \|\mathbf{A}^l -  \hat{\mathbf{A}}^l \|_2 $. From \citet{VershyninBook}, we have that $\mathbb{E}(\frac{1}{L}\|\mathbf{A}^l\|_2) \geq \frac{1}{4L}\left(\sqrt{d_1^l} + \sqrt{d_2^l}\right)$,  and $ \epsilon_l \Gamma_l $ is smaller than this in expectation for sufficiently small $\epsilon_l$. Therefore, we can use \Cref{lem:Neyshabur2017} and \Cref{lem:errortwo}. Thus we have the following 
    \begin{align}
        \|x^l - \hat{x}^l \|_2 &\leq ed_1^0 \left(  \prod_{l=1}^L L_l\|\mathbf{A}^l\|_2 \right) \sum_{l=1}^{L}\frac{\|\mathbf{A}^l - \hat{\mathbf{A}}^l\|_2}{\|\mathbf{A}^l\|_2}\nonumber\\
        &\leq ed_1^0 \left(  \prod_{l=1}^d L_l\|\mathbf{A}^l\|_2 \right) \sum_{l=1}^{L}\frac{\epsilon_l \Gamma_l}{\|\mathbf{A}^l\|_2} \nonumber
    \end{align}
\end{proof}

% \subsection{Proof of \Cref{lem:discreteerror}}
% \discreteerror*
% \begin{proof}
% Given that each atom of the discretized model is at most $\rho$ away from the pruned model and we only need to discretize the nonzero elements, we have
% \begin{align}
%     \|\tilde{\mathbf{A}}^l - \hat{\mathbf{A}}^l\|_2 &\leq \|\tilde{\mathbf{A}}^l - \hat{\mathbf{A}}^l\|_F\nonumber\\
%     &= \sqrt{\sum_{i,j} |\tilde{\mathbf{A}}_{i,j}^l - \hat{\mathbf{A}}_{i,j}^l|^2}\nonumber\\
%     &\leq \rho  J_l\nonumber
% \end{align} 
% Here, we have that $J_l$ is the number of nonzero parameters in layer $l$.  Moreover, from the triangle inequality, we have that 
% \begin{align}
%     \|\mathbf{A}^l - \tilde{\mathbf{A}}^l\|_2 \leq& \|\mathbf{A}^l - \hat{\mathbf{A}}^l \|_2 + \| \hat{\mathbf{A}}^l - \tilde{\mathbf{A}}^l\|_2\nonumber\\
%     \leq& \epsilon_l \Gamma_l +\rho_lJ_l\nonumber
% \end{align}
% Using this alongside \Cref{lem:Neyshabur2017} gets our desired bound.
% \end{proof}

\subsection{Proof of \Cref{lem:discreteerror}}
\discreteerror*
\begin{proof}
We will compare the output of the original model $x^l$ with the output of the compressed and discretized model $\tilde{x}^l$. To use the perturbation bound from \Cref{lem:Neyshabur2017}, we need that $\|\mathbf{A}^l - \tilde{\mathbf{A}}^l\|_2 \leq \frac{1}{L}\|\mathbf{A}^l\|_2$. For each layer,  we can choose a discretization parameter to satisfy this. We demonstrate this in the following: 

\begin{align}
    \|\mathbf{A}^l - \tilde{\mathbf{A}}^l\|_2 &\leq \|\mathbf{A}^l - \hat{\mathbf{A}}^l\|_2 + \|\hat{\mathbf{A}}^l - \tilde{\mathbf{A}}^l\|_2 \nonumber\\
    &\leq \epsilon_l \Gamma_l +  \rho_l  J_l\nonumber
\end{align}
Therefore, as long as we choose $$\rho_l \leq \frac{\frac{1}{L}\|\mathbf{A}^l\|_2 - \epsilon_l \Gamma_l}{J_l}\text{,}$$ we have our desired property. Therefore, using \Cref{lem:Neyshabur2017}, we have that 
\begin{align}
    \|x^l - \tilde{x}^l \|_2 &\leq ed_1^0 \left(  \prod_{l=1}^L L_l\|\mathbf{A}^l\|_2 \right) \sum_{l=1}^{h}\frac{\|\mathbf{A}^l - \tilde{\mathbf{A}}^l\|_2}{\|\mathbf{A}^l\|_2}\nonumber\\
    &\leq ed_1^0 \left(  \prod_{l=1}^d L_l\kappa_l\|\mathbf{A}^l\|_2 \right) \sum_{l=1}^{L}\frac{\|\mathbf{A}^l - \hat{\mathbf{A}}^l\|_2 + \|\hat{\mathbf{A}}^l - \tilde{\mathbf{A}}^l\|_2}{\|\mathbf{A}^l\|_2}\nonumber\\
    &\leq ed_1^0 \left(  \prod_{l=1}^d L_l\|\mathbf{A}^l\|_2 \right) \sum_{l=1}^{L}\frac{\epsilon_l \Gamma_l +  \rho_l J_l}{\|\mathbf{A}^l\|_2} \nonumber
\end{align}
This happens only if the event from \Cref{lem:errortwo} occurs for every layer. Using the union bound, we know that this happens with probability at least $1 - \sum_l^L \epsilon_l^{-1}$
\end{proof}
% \begin{proof}
% we will analyze the quantity $\norm{u^{\top}(\mathbf{\hat{A}}^l - \mathbf{A}^l)v}$ where $u$ is the vector of all $1$'s and
% $\mathbf{\hat{A}}^l - \mathbf{A}^l) = \Delta$. By \Cref{lem:meaneletwo}, we have $\mathbb{E}(u^\top\Delta v) = 0$. By \Cref{lem:varboundedtwo}, we know that $\operatorname{Var}(u^\top\Delta v) = \left(\frac{\Psi}{5^{\frac{3}{2}}} + \frac{\sqrt{5} - \sqrt{3}}{\sqrt{3}\sqrt{5}}\right)\sum_{i,j}(u_iv_j)^2$. By Chebyshev's inequality, we know that $$Pr\left( |u^\top\Delta v| \geq k\sqrt{\left(\frac{d^\frac{3}{2}\Psi}{(d+4)^\frac{3}{2}} +  \frac{\sqrt{d}(\sqrt{d+4} -\sqrt{d+2})}{\sqrt{d+2}\sqrt{d+4}}\right)\sum_{i,j}(u_iv_j)^2}\right) \leq \frac{1}{k^2}\text{.}$$
% % Now, to prove the full statement, we use
% % \begin{align}
% %     \epsilon_l \norm{A}_F\norm{u} \norm{v} \geq& \epsilon_l\sqrt{(1-t)d_1^ld_2^l\Psi} \norm{u} \norm{v}\nonumber\\
% %     \geq& \epsilon_l \sqrt{(1-t)d_1^ld_2^l\Psi} \sqrt{\sum_{i,j}(u_iv_j)^2}\nonumber\\
% %     \geq& k\sqrt{\left(\frac{d^\frac{3}{2}\Psi}{(d+4)^\frac{3}{2}} +  \frac{\sqrt{d}(\sqrt{d+4} -\sqrt{d+2})}{\sqrt{d+2}\sqrt{d+4}}\right)\sum_{i,j}(u_iv_j)^2}
% % \end{align}
% % where the first inequality comes from \Cref{lem:frobanorm} and the third inequality comes from setting $k \leq \frac{\epsilon_l\sqrt{(1-t)d_1^ld_2^l\Psi} }{\sqrt{\left(\frac{d^\frac{3}{2}\Psi}{(d+4)^\frac{3}{2}} +  \frac{\sqrt{d}(\sqrt{d+4} -\sqrt{d+2})}{\sqrt{d+2}\sqrt{d+4}}\right)}}$.
% \end{proof}
\section{Error Bound under Subgaussian Conditions}
\label{sec:augmentedperturb}
We can form tighter bounds by considering what the expected maximum of $(\hat{\mathbf{A}}^l - \mathbf{A}^l)x$ would be with high probability. If $d_2^l < d_1^l$, we observe that the matrix $\hat{\mathbf{A}}^l - \mathbf{A}^l$ has at most $d_2^l$ nonzero singular values. We need a Subgaussian condition assumption on our input to each layer to do this well to improve our bounds. 

\begin{restatable}{condition}{subgaussconcin}
    \label{con:subgaussconcin}
    The input to each layer $l \in [L]$, belongs to a distribution $\mathcal{D}$, such that for some unit magnitude vector $v$ and an arbitrary vector $x$ sampled from $\mathcal{D}$ satisfy
    $$\mathbb{P}(\langle x, v\rangle \geq t) \leq ae^{-bt^2d_1^l} \text{.}$$ Here $a$ and $b$ are universal constants greater than $0$. 
\end{restatable}

It should be noted that \Cref{con:subgaussconcin} is often seen throughout the theory of High Dimensional Statistics. The uniform distribution over the unit sphere satisfies \Cref{con:subgaussconcin}. Given this \Cref{con:subgaussconcin}, we can bound the approximation error from significantly increasing in any given layer. 

We want to do a bound on how much error the compression introduces on the margin of the training dataset. While traditional bounds assume worst-case blow-up, we can use the fact that vectors are roughly orthogonal in high-dimensional spaces. 
\begin{restatable}{lemma}{notworstcaseerror}
\label{lem:notworstcaseerror}
    Suppose we are given a matrix $\mathbf{B}$ of size $d_1^l \times d_2^l$ where $d_1^l \geq d_2^l$ and $\mathcal{S}$ is a collection of vectors from a distribution satisfying \Cref{con:subgaussconcin}. For any vector $x \in \mathcal{S}$,  there exists constants $a, b$ such that 
    $$\|\mathbf{B}x\|_2 \leq \sqrt{d_2^l}t_l\|\mathbf{B}\|_2\|x\|_2,$$ with probability at least $1 - |\mathcal{S}|a e^{-bt_l^2d_1^l}$. We will call $\kappa_l = \sqrt{d_2^l}t_l$ if $d_2^l \leq d_1^l$ and $\kappa_l = 1$ otherwise.
\end{restatable}
\begin{proof}
    We first decompose $\mathbf{B} = U\Sigma V$ using Singular Value Decomposition. 
    Therefore, for any $x$ we have that,
    \begin{align}
        \|\mathbf{B}x\|_2 &= \|U\Sigma Vx\|_2\nonumber\\
        &=  \|\Sigma Vx\|_2\nonumber\\
        &= \|\Sigma y\|_2\nonumber.
    \end{align}
    The second equality comes from the fact that $U$ is unitary and norm-preserving, and the third equality comes from setting $y = Vx$. Now, if $x$ is some standard random normal vector, then $y$ too is a standard random normal vector. We also observe that $\Sigma$ is a diagonal matrix where only the first $d_2^l$ values are nonzero. We use the well-known identity that if $v$ is a vector with a magnitude of $1$,
    $$\mathbb{P}(\langle v, y \rangle \geq t) \leq ae^{-bt^2d_1^l}\text{.}$$
    Here, $a$ and $b$ are global constants. Therefore, applying this inequality for the respective nonzero singular values in $\Sigma$, we have
    $$\mathbb{P}\left(\| \Sigma y\|_2 \geq \sqrt{d_2^l}t\|\mathbf{B}\|_2\right) \leq ae^{-bt^2d_1^l}\text{,}$$
    since $\|B\|_2$ is the maximum singular value. Applying the union bound for each element of $\mathcal{S}$, we have that for every element in $\mathcal{S}$
    $$\|\mathbf{B}x\|_2 \leq \sqrt{d_2^l}t\|\mathbf{B}\|_2\|x\|_2,$$ with probability at least $1 - |\mathcal{S}|a e^{-bt^2d_1^l}$.
\end{proof}

\begin{restatable}{lemma}{Neyshabur2017subgaussian}
    \label{lem:Neyshabur2017subgaussian}
    For the weights of the model $\mathbf{M}$ and any perturbation $\mathbf{U}^l$ for $l \in [h]$ where the perturbed layer $l$ is $\mathbf{U}^l + \mathbf{A}^l $, given that $\|\mathbf{U}^l\|_2 \leq \frac{1}{L} \|\mathbf{A}^l\|_2$, we have that for all input $x^0 \in \mathcal{S}$, 
    $$\|x^L - \bar{x}^L \|_2 \leq ed_1^0 \left(  \prod_{l=1}^L \kappa_lL_l\|\mathbf{A}^l\|_2 \right) \sum_{l=1}^{L}\frac{\|\mathbf{U}^l\|_2}{\|\mathbf{A}^l\|_2}\text{.}$$ Here, $\bar{x}^L$ denotes the output of the $L$th layer of the perturbed model. This happens if \Cref{con:subgaussconcin} occurs.
\end{restatable}
\begin{proof}
    This proof mainly follows from \citet{Neyshabur2017}. We restate it here with the differing notation for clarity and completeness. We will prove the induction hypothesis that $\|\bar{x}^l - x^l\|_2 \leq \left(1 + \frac{1}{L}\right)^l \|x^0\|_2 \left(  \prod_{i=1}^l \kappa_iL_i\|\mathbf{A}^l\|_2 \right) \sum_{i=1}^{l}\frac{\|\mathbf{U}^i\|_2}{\|\mathbf{A}^i\|_2}$. The base case of induction trivially holds, given we have that $\|\bar{x}^0 - x^0\|_2 = 0$ by definition. Now, we prove the induction step. Assume that the induction hypothesis holds for $l$. We will prove that it holds for $l+1$. 
    We have that
    \begin{align}
        \|x^l - \bar{x}^l\|_2 &\leq \|\left(\mathbf{A}^l + \mathbf{U}^l\right)\phi_l(\bar{x}^{l-1}) - \mathbf{A}^l\phi_l(x^{l-1})\|_2\nonumber\\
        &\leq \|\left(\mathbf{A}^l + \mathbf{U}^l\right)(\phi_l(\bar{x}^{l-1}) - \phi_l(x^{l-1})) + \mathbf{U}^l\phi_l(x^{l-1})\|_2\nonumber\\
        &\leq \left(\|\mathbf{A}^l\|_2 + \|\mathbf{U}^l\|_2\right)\|\phi_l(\bar{x}^{l-1}) - \phi_l(x^{l-1})\|_2 + \|\mathbf{U}^l\|_2\|\phi_l(x^{l-1})\|_2\label{eq:applykappasub}\\
        &\leq \left(\|\mathbf{A}^l\|_2 + \|\mathbf{U}^l\|_2\right)\|\phi_l(\bar{x}^{l-1}) - \phi_l(x^{l-1})\|_2 + \|\mathbf{U}^l\|_2\|\phi_l(x^{l-1})\|_2\nonumber\\
        &\leq L_l  \left(\|\mathbf{A}^l\|_2 + \|\mathbf{U}^l\|_2\right)\|\bar{x}^{l-1} - x^{l-1}\|_2 + L_l \|\mathbf{U}^l\|_2\|x^{l-1}\|_2 \label{eq:ney2sub}\\
        &\leq L_l  \left(1 +\frac{1}{d}\right)\left(\|\mathbf{A}^l\|_2\right)\|\bar{x}^{l-1} - x^{l-1}\|_2 + L_l \|\mathbf{U}^l\|_2\|x^{l-1}\|_2\nonumber\\
        &\leq L_l  \left(1 +\frac{1}{d}\right)\left(\|\mathbf{A}^l\|_2\right)\left(1 + \frac{1}{L}\right)^{l-1} \|x^0\|_2 \left(  \prod_{i=1}^{l-1} L_i \kappa_i\|\mathbf{A}^i\|_2 \right) \sum_{i=1}^{l-1}\frac{\|\mathbf{U}^i\|_2}{\|\mathbf{A}^i\|_2} + L_l \|\mathbf{U}^l\|_2\|x^{l-1}\|_2 \label{eq:ney1sub}\\
        &\leq L_l  \left(1 +\frac{1}{L}\right)^{l}\left(  \prod_{i=1}^{l-1} L_i \kappa_i \|\mathbf{A}^i\|_2 \right) \sum_{i=1}^{l-1}\frac{\|\mathbf{U}^i\|_2}{\|\mathbf{A}^i\|_2}\|x^0\|_2  + L_l \|x^0\|_2 \|\mathbf{U}^l\|_2\prod_{i=1}^{l-1}L_{i}\|\mathbf{A}^i\|_2\nonumber\\
        &\leq L_l  \left(1 +\frac{1}{L}\right)^{l}\left(  \prod_{i=1}^{l-1} L_i \kappa_i \|\mathbf{A}^i\|_2 \right) \sum_{i=1}^{l-1}\frac{\|\mathbf{U}^i\|_2}{\|\mathbf{A}^i\|_2}\|x^0\|_2 + \|x^0\|_2 \frac{\|\mathbf{U}^l\|_2}{\|\mathbf{A}^l\|_2}\prod_{i=1}^{l}L_{i}\|\mathbf{A}^i\|_2\nonumber\\
        &\leq \left(1 +\frac{1}{L}\right)^{l}\left(  \prod_{i=1}^{l} L_i \kappa_i \|\mathbf{A}^i\|_2 \right) \sum_{i=1}^{l}\frac{\|\mathbf{U}^i\|_2}{\|\mathbf{A}^i\|_2}\|x^0\|_2 \nonumber
    \end{align}
    Here, \Cref{eq:applykappasub} results from applying \Cref{lem:notworstcaseerror}. \Cref{eq:ney2sub} comes from the fact that $\phi_i$ is $L_i$-Lipschitz smooth and that $\phi_i(0) = 0$. Moreover, \Cref{eq:ney1sub} comes from applying the induction hypothesis. Therefore, we have proven the induction hypothesis for all layers. We now only need the fact that $\left(1 + \frac{1}{L}\right)^L \leq e$, and we have our final statement. If \Cref{con:subgaussconcin} is not satisfied, we need only set $\kappa_l = 1 $ for all $l \in [L]$ and the analysis will remain valid. 
\end{proof}
\subsection{Proof of \Cref{lem:finalerrorboundsubgaussian}}
 \begin{restatable}{lemma}{finalerrorboundsubgaussian}
    \label{lem:finalerrorboundsubgaussian}
    The difference between outputs of the pruned model and the original model on any input $x$ is bounded by, with probability at least  $1 - \left[\sum_i^L \epsilon_i^{-1} +  |\mathcal{S}|a e^{-bt_l^2d_1^l}\right]$,
    $$\|\hat{x}^L - x^L\| \leq  ed_1^0 \left(  \prod_{l=1}^L L_l\kappa_l\|\mathbf{A}^l\|_2 \right) \sum_{l=1}^{L}\frac{\epsilon_l \Gamma_l}{\|\mathbf{A}^l\|_2}\text{.}$$ Here, $a,b$ are positive constants from the distribution of input data. 
\end{restatable}
\begin{proof}
  We will compare the output of the original model $x^l$ with the output of the compressed model. We need the fact that $\frac{1}{L} \|\mathbf{A}^l\|_2 \geq \epsilon_l \Gamma_l \geq \|\mathbf{A}^l -  \hat{\mathbf{A}}^l \|_2 $. From \citet{VershyninBook}, we have that $\mathbb{E}(\frac{1}{L}\|\mathbf{A}^l\|_2) \geq \frac{1}{4L}\left(\sqrt{d_1^l} + \sqrt{d_2^l}\right)$,  and $ \epsilon_l \Gamma_l $ is smaller than this in expectation for sufficiently small $\epsilon_l$. Therefore, we can use \Cref{lem:Neyshabur2017subgaussian}. Thus we have the following 
    \begin{align}
        \|x^l - \hat{x}^l \|_2 &\leq ed_1^0 \left(  \prod_{l=1}^L L_l\kappa_l\|\mathbf{A}^l\|_2 \right) \sum_{l=1}^{L}\frac{\|\mathbf{A}^l - \hat{\mathbf{A}}^l\|_2}{\|\mathbf{A}^l\|_2}\nonumber\\
        &\leq ed_1^0 \left(  \prod_{l=1}^d L_l\kappa_l\|\mathbf{A}^l\|_2 \right) \sum_{l=1}^{L}\frac{\epsilon_l \Gamma_l}{\|\mathbf{A}^l\|_2} \nonumber
    \end{align}
\end{proof}

% \subsection{Proof of \Cref{lem:discreteerror}}
% \discreteerror*
% \begin{proof}
% Given that each atom of the discretized model is at most $\rho$ away from the pruned model and we only need to discretize the nonzero elements, we have
% \begin{align}
%     \|\tilde{\mathbf{A}}^l - \hat{\mathbf{A}}^l\|_2 &\leq \|\tilde{\mathbf{A}}^l - \hat{\mathbf{A}}^l\|_F\nonumber\\
%     &= \sqrt{\sum_{i,j} |\tilde{\mathbf{A}}_{i,j}^l - \hat{\mathbf{A}}_{i,j}^l|^2}\nonumber\\
%     &\leq \rho  J_l\nonumber
% \end{align} 
% Here, we have that $J_l$ is the number of nonzero parameters in layer $l$.  Moreover, from the triangle inequality, we have that 
% \begin{align}
%     \|\mathbf{A}^l - \tilde{\mathbf{A}}^l\|_2 \leq& \|\mathbf{A}^l - \hat{\mathbf{A}}^l \|_2 + \| \hat{\mathbf{A}}^l - \tilde{\mathbf{A}}^l\|_2\nonumber\\
%     \leq& \epsilon_l \Gamma_l +\rho_lJ_l\nonumber
% \end{align}
% Using this alongside \Cref{lem:Neyshabur2017} gets our desired bound.
% \end{proof}

\begin{restatable}{lemma}{discreteerrorsubgaussian}
    \label{lem:discreteerrorsubgaussian}
    The norm of the difference between the pruned layer and the discretized layer is upper-bounded as 
    $\|\tilde{\mathbf{A}}^l - \hat{\mathbf{A}}^l\|_2 \leq \rho_l J_l$ where $J_l$ is the number of nonzero parameters in $\hat{\mathbf{A}}^l$ ($J_l$ is used for brevity here and will be analyzed later). \label{lem:errorafterdiscretization}
    With probability at least $1- \left[\sum_{l=1}^L \epsilon_l^{-1} + |\mathcal{S}|a e^{-bt_l^2d_1^l}\right]$, given that the parameter $\rho_l$ for each layer is chosen such that $\rho_l \leq \frac{\frac{1}{L}\|\mathbf{A}^l\|_2 - \epsilon_l \Gamma_l}{J_l}$, we have that the error induced by both discretization and the pruning is bounded by
    $$\|x^L - \tilde{x}^L\|_2 \leq ed_1^0 \left(  \prod_{l=1}^L L_l\kappa_l\|\mathbf{A}^l\|_2 \right) \sum_{l=1}^{L}\frac{\epsilon_l \Gamma_l +  \rho_l J_l}{\|\mathbf{A}^l\|_2}\text{.}$$ 
    
\end{restatable}
\begin{proof}
We will compare the output of the original model $x^l$ with the output of the compressed and discretized model $\tilde{x}^l$. To use the perturbation bound from \Cref{lem:Neyshabur2017subgaussian}, we need that $\|\mathbf{A}^l - \tilde{\mathbf{A}}^l\|_2 \leq \frac{1}{L}\|\mathbf{A}^l\|_2$. For each layer,  we can choose a discretization parameter to satisfy this. We demonstrate this in the following: 

\begin{align}
    \|\mathbf{A}^l - \tilde{\mathbf{A}}^l\|_2 &\leq \|\mathbf{A}^l - \hat{\mathbf{A}}^l\|_2 + \|\hat{\mathbf{A}}^l - \tilde{\mathbf{A}}^l\|_2 \nonumber\\
    &\leq \epsilon_l \Gamma_l +  \rho_l  J_l\nonumber
\end{align}
Therefore, as long as we choose $$\rho_l \leq \frac{\frac{1}{L}\|\mathbf{A}^l\|_2 - \epsilon_l \Gamma_l}{J_l}\text{,}$$ we have our desired property. Therefore, using \Cref{lem:Neyshabur2017}, we have that 
\begin{align}
    \|x^l - \tilde{x}^l \|_2 &\leq ed_1^0 \left(  \prod_{l=1}^L L_l\kappa_l\|\mathbf{A}^l\|_2 \right) \sum_{l=1}^{h}\frac{\|\mathbf{A}^l - \tilde{\mathbf{A}}^l\|_2}{\|\mathbf{A}^l\|_2}\nonumber\\
    &\leq ed_1^0 \left(  \prod_{l=1}^d L_l\kappa_l\|\mathbf{A}^l\|_2 \right) \sum_{l=1}^{L}\frac{\|\mathbf{A}^l - \hat{\mathbf{A}}^l\|_2 + \|\hat{\mathbf{A}}^l - \tilde{\mathbf{A}}^l\|_2}{\|\mathbf{A}^l\|_2}\nonumber\\
    &\leq ed_1^0 \left(  \prod_{l=1}^d L_l\kappa_l\|\mathbf{A}^l\|_2 \right) \sum_{l=1}^{L}\frac{\epsilon_l \Gamma_l +  \rho_l J_l}{\|\mathbf{A}^l\|_2} \nonumber
\end{align}
This happens only if the event from \Cref{lem:errortwo} and \Cref{lem:notworstcaseerror} occur for every layer. Using the union bound, we know that this happens with probability at least $1 - \left[\sum_l^L \epsilon_l^{-1}  - |\mathcal{S}|a e^{-bt_l^2d_1^l}\right]$.
\end{proof}
\section{Naive Generalization Proofs}
Given the Gaussian assumption, it is natural to count the number of possible outcomes of the compression algorithm by counting the number of possible configurations of nonzero atoms in any matrix and then counting the possible values each atom could take after quantization. We provide the generalization bound from this intuition.

\begin{restatable}{lemma}{badgeneralization}
\label{lem:badgeneralization}
Using the counting arguments above yields a generalization bound 
$$R_0(g_A) \leq \hat{R}_\gamma(f) + \mathcal{O}\left(\sqrt{\frac{\sum_l \log({d_1^ld_2^l \choose \alpha}) +  \alpha \log \frac{1}{\rho_l}}{n  }}\right)\text{.}$$ This holds when $d$ is chosen such that $\gamma \geq ed_1^0 \left(  \prod_{l=1}^L L_l\|\mathbf{A}^l\|_2 \right) \sum_{l=1}^{L}\frac{\epsilon_l \Gamma_l +  \rho_l  J_l}{\|\mathbf{A}^l\|_2}\text{.}$
\end{restatable}

We now provide the requisite knowledge to prove this bound. We first analyze a naive methodology for counting the number of possible outcomes from the learning algorithm and compression scheme. We first provide a slightly altered generalization bound to fit our use case better. 

\begin{restatable}{theorem}{editedarora}

\label{thm:editedarora}
    If there are $J$ different parameterizations, the generalization error of a compression $g_a$ is, with probability at least $1-\delta$, $$L_0(g_A) \leq \hat{L}_{\gamma}(f) + \sqrt{\frac{\ln\left(\sqrt{\frac{J}{\delta}}\right)}{n}}\text{.}$$
\end{restatable}
\begin{proof}
Most of this proof follows from Theorem 2.1 from \citet{Arora2018}. For each $A \in \mathcal{A}$, the training loss $\hat{R}_0(g_A)$ is the average of $n$ i.i.d. Bernoulli Random variables with expected value equal to $R_0(g_A)$. Therefore, by standard Chernoff bounds, we have that,
$$\mathbb{P}(R_0(g_A) - \hat{R}_0(g_A) \geq \tau) \leq \exp(-2\tau^2n)\text{.}$$ Given $f$ is $(\gamma,\mathcal{S})$-compressible by $g_A$, we know the empirical margin loss of $g_A$ for margin $0$ is less than the empirical margin loss of $f$ with margin $\gamma$, i.e. $\hat{R}_0(g_A) \leq  \hat{R}_{\gamma}(f)$. Given there are $J$ different parameterizations, by union bound, with probability at least $1 - J\exp(-2\tau n)$, we have $R_0(g_A) \leq \tau + \hat{R}_0(g_A)$. Setting $J\exp(-2\tau n) = \delta$, we have $\tau = \sqrt{\frac{\ln\left(\sqrt{\frac{J}{\delta}}\right)}{n}}$. Therefore, with probability $1 - \delta$, we have 
$$R_0(g_A) \leq \hat{R}_{\gamma}(f) + \sqrt{\frac{\ln\left(\sqrt{\frac{J}{\delta}}\right)}{n}}\text{.}$$ 
\end{proof}

Now, we can state the number of parameterizations achievable by our compression algorithm. If there are $d_1^ld_2^l$ elements in the matrix and $\alpha$ stays nonzero after compression, then there are ${d_1^ld_2^l \choose \alpha}$ total parameterizations for each layer. Moreover, within each parameterization, there are $r^\alpha$ ways to choose the values that each nonzero element takes given each of the $\alpha$ atoms can take $r$ values where $r = \mathcal{O}\left(\frac{1}{\rho_l}\right)$. We, therefore, need a bound on the number of elements that stay nonzero after pruning. We achieve this with the following two lemmas. We will first prove that at least $\tau$ elements have probability $\kappa$ of getting compressed. Using such a counting argument yields the following generalization bound.
\begin{restatable}{lemma}{atleastkappatwo}
\label{lem:atleastkappatwo}
At least $\tau$ elements of a given matrix $\mathbf{A}^l$ will have a probability of at least $\kappa$ of getting compressed. This event occurs with probability at least $1 - I_{1 - p_1} \left(d_1d_2 - \tau, 1 + \tau\right)$ where $p_1=\operatorname{erf}\left(\sqrt{\frac{-d\ln(\kappa)}{2}}\right)$. Here, $\operatorname{erf}$ is the Gauss Error Function.
\end{restatable}

\begin{proof}
    For any given element to have a probability of at least $\kappa$ of getting compressed, $$\operatorname{exp}\left(\frac{-\mathbf{A}_{i,j}^2}{d\Psi}\right) \geq \kappa\text{.}$$

    This means that $$|\mathbf{A}_{i,j}| \leq  \sqrt{-d\Psi\ln(\kappa)}\text{.}$$
    Given that $|\mathbf{A}_{i,j}|$ follows a Folded Normal Distribution, the probability of this happening is 

    \begin{align}
        p_1 &= \mathbb{P}\left(|\mathbf{A}_{i,j}| \leq  \sqrt{-d\Psi\ln(\kappa)}\right)\nonumber\\
        &= \operatorname{erf}\left(\frac{\sqrt{-d\Psi\ln(\kappa)}}{\sqrt{2\Psi}}\right)\nonumber\\
        &= \operatorname{erf}\left(\sqrt{\frac{-d\ln(\kappa)}{2}}\right)
    \end{align}
    For notational ease, we denote the set of atoms that satisfy this criterion  $\mathcal{C} = \left\{(i,j) | \operatorname{exp}\left(\frac{-\mathbf{A}_{i,j}^2}{d\Psi}\right) \geq \kappa\right\}$.
    Therefore, the number of elements $\tau$ that will have the probability of getting compressed larger than $\kappa$ obeys a binomial distribution. Therefore, 
    $$\mathbb{P}(|\mathcal{C}| \geq \tau) = 1 - I_{1 - p_1} \left(d_1d_2 - \tau, 1 + \tau\right)\text{.}$$
    Here, $I$ is the Regularized Incomplete Beta Function. 
\end{proof}
Using this probabilistic upper bound from \Cref{lem:atleastkappatwo}, we can upper bound the number of nonzero elements in any matrix. 
\begin{restatable}{lemma}{atleastalphatwo}
\label{lem:atleastalphatwo}
Given the event from \Cref{lem:atleastkappatwo} happens, the probability that at least $\alpha$ elements will end up being compressed is at least $1 - I_{1-\kappa}\left(\tau - \alpha, \alpha + 1 \right)$.
    
\end{restatable}
\begin{proof}
There are at least $\tau$ elements with probability greater than $\kappa$. In the worst case, the other $d_1^ld_2^l - \tau$ elements are not compressed. The probability distribution over the remaining elements is a binomial distribution with probability $\kappa$. Therefore, the probability that at least $\alpha$ of the $\tau$ elements are compressed is at least $1 - I_{1-\kappa}\left(\tau - \alpha, \alpha + 1 \right)$.
\end{proof}
Now, we can finally prove our naive generalization bound. 

\begin{proof}
  From \Cref{thm:editedarora}, we have $$L_0(g_A) \leq \hat{L}_{\gamma}(f) + \sqrt{\frac{\ln\left(\sqrt{\frac{J}{\delta}}\right)}{n}},$$ where $J$ is the number of parameter configurations. Each matrix has ${d_1^ld_2^l \choose \alpha}$ different ways to arrange the nonzero elements. Within any such configuration, there are $r^\alpha$ ways to select the values for any of the nonzero elements, where $r_l = \mathcal{O}\left(\frac{1}{\rho_l}\right)$ is the number of values any atom could take after discretization. This yields a generalization bound of 
    $$R_0(g_A) \leq \hat{R}_\gamma(f) + \mathcal{O}\left(\sqrt{\frac{\log({d_1^ld_2^l \choose \alpha}) +  \alpha \log r_l}{n}}\right)\text{.}$$
    Here, we only require that $\gamma \geq ed_1^0 \left(  \prod_{l=1}^d L_l\kappa_l\|\mathbf{A}^l\|_2 \right) \sum_{l=1}^{L}\frac{\epsilon_l \Gamma_l +  \rho_l  J_l}{\|\mathbf{A}^l\|_2} $ given \Cref{lem:errorafterdiscretization}.
\end{proof}
Regrettably, such a bound is quite poor in its dependence on the size of the matrices, mainly due to the logarithm term of the factorial, which is a significantly large value. This is, in fact, worse than many of the previous bounds in the literature. This is due to the combinatorial explosion of the number of sparse matrices. However, if there exists a way to instead represent the space of sparse matrices within the space of dense matrices of much smaller dimensions, then we could avoid such a combinatorial explosion of the number of parameters. This is the exact purpose of matrix sketching.

\section{Matrix Sketching Proofs}
\subsection{How to choose $A, B$}
\label{sec:chooseab}
 To generate $A$ or $B$, we can first sample a random bipartite graph where the left partition is of size $m$, and the right partition is of size $p_1$ or the dimension of the matrix to be sketched. If we say that any node in the left partition is connected to at most $\delta$ nodes, we can call this bipartite graph a $\delta$-random bipartite ensemble. We have the resulting definition below.
\begin{definition}
    $G$ is a bipartite graph $G = ([x],[y], \mathcal{E})$ where $x$ and $y$ are the size of the left and right partitions, respectively, and $\mathcal{E}$ is the set of edges. We call $G$ a $\delta$-random bipartite ensemble if every node in $[x]$ is connected to most $\delta$ nodes in $[y]$ and each possible connection is equally likely. 
\end{definition}
Given this setup, we can choose the matrices $A$ and $B$ as adjacency matrices from a random $\delta$-random bipartite ensemble. Intuitively, such $A$ and $B$ are chosen such that any row in $A$ or $B$ has at most $\delta$ $1$'s. Any given element of $Y_{ij}$ is $\sum_{kl}A_{ik}{\tilde X}_{kl}B_{lj}$. However, only approximately $\delta^2$ of the elements in the sum are nonzero. Therefore, $Y_{ij}$ is expressed as the sum of $\delta^2$ terms from the sum $\sum_{kl}A_{ik}{\tilde X}_{kl}B_{lj}$. We can then express many elements from $Y$ by changing which elements are set or not set to zero in this sum. This gives a visual explanation of how this sketching works. Furthermore, the power of the expressiveness of the sketching depends mainly on the parameters $m$ and $\delta$. Here, we can bound the size required for $m$ and $\delta$ such that the solution to \Cref{eq:prob1} leads to one-to-one mapping with high probability. 
\subsection{Remaining Proofs}
Given that each of the atoms is identically distributed and independent, given $N$ atoms are not pruned, the problem of characterizing how these atoms are distributed among the rows or columns is similar to the famous balls-and-bins problem. We provide a helper lemma to prove that our pruning method generates a distributed sparse matrix. We use the famous lemma from \citet{Richa2000}. 
\begin{restatable}{lemma}{richa}
    \label{lem:richa}
    Consider the problem of throwing $N$ balls independently a uniformly at random into n bins. Let $X_j$ be the random variable that counts the number of balls in the $j$-th bin. With probability at least $1 - n^{-\frac{1}{3}}$, we have that $$\underset{j}{\max} X_j \leq \frac{3N}{n}\text{.}$$
\end{restatable}
We now use this lemma to prove our distributed sparsity. 
\subsection{Proof of \Cref{lem:algdisspar}}
\algdisspar*
\begin{proof}
    We will first prove a bound on the number of noncompressed atoms, a random variable we will call $N$. The probability that any given element gets pruned is 
    \begin{align}
        \mathbb{P}(Z_{i,j} = 0) =& \int_{-\infty}^{\infty}\mathbb{P}(Z_{i,j} = 0|\mathbf{A}_{i,j}^l)\cdot \mathbb{P}(\mathbf{A}_{i,j}^l)d\mathbf{A}_{i,j}^l\\
        =& \int_{-\infty}^{\infty}\operatorname{exp}\left(\frac{-(\mathbf{A}_{i,j}^l)^2}{d\Psi} \right)\frac{1}{\sqrt{2\pi \Psi}} \operatorname{exp}\left(\frac{-1}{2}\frac{(\mathbf{A}_{i,j}^l)^2}{\Psi}\right)d\mathbf{A}_{i,j}^l\\
        =& \frac{\sqrt{d+2}  - \sqrt{d}}{\sqrt{d+2}}
    \end{align}
    
    Therefore, the expected number of nonzero elements after pruning is $\mathbb{E}(N) = \frac{d_{1}^ld_2^l(\sqrt{d+2}  - \sqrt{d})}{\sqrt{d+2}}$.
    Using Markov's inequality, we have that 
    $$\mathbb{P}(N \geq t) \leq \frac{\mathbb{E}(N)}{t}\text{.}$$
    Here, we set $t = \lambda_i \frac{d_{1}^ld_2^l(\sqrt{d+2}  - \sqrt{d})}{\sqrt{d+2}}$. Using this, we have with probability at least $1 - \frac{1}{\lambda_i}$,
    $$N \leq  \lambda_i \frac{d_{1}^ld_2^l(\sqrt{d+2}  - \sqrt{d})}{\sqrt{d+2}}\text{.}$$
    Here, we can use \Cref{lem:richa}. For the rows, with probability at least $1  - (d_1^l)^{\frac{-1}{3}}$, we have that the maximum number of nonpruned atoms in any row is at most $$\frac{3N}{d_1^l} = 3 \lambda_i \frac{d_2^l(\sqrt{d+2}  - \sqrt{d})}{\sqrt{d+2}}\text{.}$$ Similarly, we have that the maximum number of nonpruned atoms in any column is at most $$\frac{3N}{d_2^l} = 3 \lambda_i \frac{d_1^l(\sqrt{d+2}  - \sqrt{d})}{\sqrt{d+2}}\text{.}$$ Therefore, we have that this occurs with probability at least $1 - \frac{1}{\lambda_i} -(d_1^l)^{-\frac{1}{3}} - (d_2^l)^{-\frac{1}{3}}$. 
\end{proof}

\subsection{Proof of \Cref{thm:generrorsketch}}
\generrorsketch*
\begin{proof}
    From \Cref{lem:algdisspar}, we know that $\max(j_r, j_c) \leq 3 \lambda_i \frac{\max(d_2^l, d_1^l)(\sqrt{d+2}  - \sqrt{d})}{\sqrt{d+2}}$. Therefore, we can compress any matrix $\mathbf{A}^l$ into a sparse matrix $\hat{\mathbf{A}}^l$ and then further into a small matrix of size $(\sqrt{j_lp_l}\log(p_l))^2$ from \Cref{thm:dasarthy}.Therefore, we have that $$(\sqrt{j_lp_l}\log(p_l))^2 \leq  3 \lambda_i \frac{d_2^ld_1^l(\sqrt{d+2}  - \sqrt{d})}{\sqrt{d+2}}\log^2(p_l)\text{.}$$ By \Cref{thm:aroraoriginal}, we have that 
    $$L_0(g_A) \leq \hat{L}_\gamma(f) + \mathcal{O}\left(\sqrt{\frac{\sum_l 3 \lambda_i \frac{d_2^ld_1^l(\sqrt{d+2}  - \sqrt{d})}{\sqrt{d+2}}\log^2(p_l) \log(\frac{1}{\rho_l})}{n}}\right)\text{.}$$
\end{proof}

\subsection{Proof of \Cref{thm:impgen}}
\impgen*
\begin{proof}
    Proving a generalization bound using our framework usually includes one, proving the error due to compression is bounded, and two, obtaining a bound on the number of parameters. \citet{Malach2020} fortunately proves both. We restate the bound from \citet{Arora2018}: 
     $$R_0(\tilde{G}) \leq \hat{R}_\gamma(\tilde{G}) + \mathcal{O}\left(\sqrt{\frac{q \log r}{n}}\right)\text{.}$$ From \Cref{thm:wtsubnetwork}, we have that $$
\sup_{x \in \mathcal{X}}|F(x)-\tilde{G}(x)| \leq \epsilon
\text{.}$$ Directly setting $\gamma = \epsilon + \epsilon_{\rho}$ satisfies our error requirement, where $\epsilon_{\rho}$ is the small error introduced due to discretization. Now, we must focus on bounding the number of parameters in the model. Fortunately, \citet{Malach2020} provides a useful bound. They show that the first layer has approximately $\mathcal{O}(D_Fd_0^1)$ nonzero parameters, and the rest of the layers of $\tilde{G}$ have approximately $\mathcal{O}(D_F^2)$ nonzero parameters. Moreover, from the proof of Theorem 2.1, they show that these nonzero parameters are evenly distributed across rows and columns. Therefore, we can use our matrix sketching framework to show that we can compress the set of outputs from Iterative Pruning to a smaller, dense set of matrices. Namely, the middle layers of $\tilde{G}$ such as $W_i^{\tilde{G}}$ can be represented as a smaller matrix of dimension $m = \mathcal{O}(D_F\log(D_G))$ from \Cref{thm:dasarthy}. For the first layer, we can also use matrix sketching to represent it as a matrix of size $\mathcal{O}(\sqrt{D_Fd_{0,1}}\log(D_G))$. We now have an appropriate bound on the number of parameters in our sketched models. We apply trivial discretization by rounding the nearest value of $\rho$. Therefore, we have from \citet{Arora2018} 
$$R_0(\tilde{G}) \leq \hat{R}_{\epsilon + \epsilon_{\rho}}(\tilde{G}) + \mathcal{O}\left(\sqrt{\frac{\left[D_Fd_{0,1}\log(D_G)^2 + LD_F^2\log(D_G)^2\right]\log\left(\frac{1}{\rho}\right)}{n} }\right)\text{.}$$
We can apply the matrix sketching to each of the $L$ rows with probability at least $1 - D_G^{-c}$ according to \Cref{thm:dasarthy}. The error of the pruned model is also bounded by $\epsilon$ with at least probability $1 - \delta$. Union bounding these together show that this bound holds with probability at least $1-\delta - LD_G^{-c}$.
\end{proof}

\section{Additional Empirical Results}
We show the detailed empirical results on the MNIST and CIFAR10 datasets in \Cref{tab:mnist} and \Cref{table:cifar} respectively, and are supplemental to the results obtained in Section \ref{sec:experiments}. All bounds are shown on a logarithmic scale. We compare our bounds with some standard norm-based generalization bounds of \citet{neyshabur2015norm},\citet{Bartlett2017}, and \citet{Neyshabur2017}. For comparing our bound on MNIST, we use an MLP with hidden dimensions 500, 784, 1000, 1500, 2000, and 2500 where the depth is kept constant. The model training details are detailed in \Cref{sec:experiments}. We see that across different hidden dimensions, our generalization bounds are consistently better than other generalization bounds. Over different hidden dimensions, the true generalization error seems to remain relatively stable. Relative to other bounds, our generalization bound seems more stable than other bounds, increasing at a lesser rate than other bounds as the hidden dimension increases. However, we unfortunately do not capture the seeming independence between the hidden dimension and true generalization error. For our bound, this is due to the fact that the margin of the trained model is not increasing enough with the increase in model size. Our bound predicts the generalization error of pruning in terms of the margin of the original model. If the margin of the original model does not increase while the model's size increases, our bound will increase. Therefore, this bound needs more information to capture generalization more accurately.

Additionally, we show the dependence of our bound on the number of training epochs in Figure \ref{fig:overtraining}, where we take the original MLP of depth $5$ and compare how our generalization bound and the true generalization error change over epochs. It is to be noted that our bound is scaled to be in the same range as the true generalization error. There are differences between the curves, indicating our bound needs to include additional information needed to explain generalization fully. Our bound does decrease over time as the margin increases, mimicking the true generalization error. The main interesting difference is that the downward curve for our bound occurs in two places. The first drop in our generalization bound happens only because of the drop of the generalization error, but the margin is still negative. Once the margin becomes positive and increases, our bound slowly begins to decrease. At this point, however, the true generalization error seems to have already reached its minimum. 

In \Cref{table:cifar}, all the insights noticed on the MNIST dataset seem to extend to CIFAR10. Our generalization bound is tighter than existing state–of–the–art norm-based generalization bounds. Indeed our error is orders of magnitude tighter than other generalization bounds. We note that while all generalization bounds here are far worse than the MNIST counterparts, our generalization bound most accurately reflects the true jump in generalization error between MNIST and CIFAR10. For both ours and the true generalilzatiion error, the bounds differ by one order of magnitude between MNIST and CIFAR10. However, the other bounds differ by at least $7$ orders of magnitude. Our bound seems to capture more of the behavior of the true generalization error than these other bounds in this regard. 
\begin{table}[h]
    \centering
        \begin{tabular}{lllllll}
    \toprule
     {\centering METHOD}    & \multicolumn{1}{p{1cm}}{\centering MNIST \\ 500}  & \multicolumn{1}{p{1cm}}{\centering MNIST \\ 784} & \multicolumn{1}{p{1cm}}{\centering MNIST \\ 1000} & \multicolumn{1}{p{1cm}}{\centering MNIST \\ 1500}  &\multicolumn{1}{p{1cm}}{\centering MNIST \\ 2000}  & \multicolumn{1}{p{1cm}}{\centering MNIST \\ 2500} \\
    \midrule
    Neyshabur 2015 & 22.29 & 23.56 & 24.42 & 25.12 & 27.03 & 27.72    \\
    Neyshabur 2017 & 17.91 & 18.34 & 18.70 & 18.81 & 21.50 &  21.57      \\
    Bartlett 2017  & 11.51  & 11.68 & 11.87 & 11.70 & 13.96 & 13.81 \\
    Ours     & 3.36  & 3.77 & 4.00 & 4.40 & 4.73 & 4.96  \\
    True error & -3.76  & -3.84  & -3.80 & -3.85 & -3.86 & -3.87   \\
    \bottomrule
  \end{tabular}
    \caption{Generalization bounds on logarithmic scale w.r.t. MNIST using MLP of varying dimensions.}
      \label{tab:mnist}
  \centering
    \label{tab:mnist}
    
\end{table}

\begin{table}[H]
\centering
\begin{tabular}{ll}
    \toprule
     METHOD    & CIFAR10 \\
    \midrule
    Neyshabur 2015 & 33.19   \\
    Neyshabur 2017 &  30.10     \\
    Bartlett 2017  &  22.40 \\
    Ours     & 4.68 \\
    True error & -2.41 \\
    \bottomrule
  \end{tabular}
    \caption{Comparison of different generalization bounds on the CIFAR10 dataset on a logarithmic scale}
    \label{table:cifar}
 \centering
    \label{tab:cifar}
\end{table}

% \begin{table}[H]
% \begin{minipage}[b]{0.5\linewidth}
% \centering
% \begin{tabular}{ll}
%     \toprule
%      METHOD    & CIFAR10 \\
%     \midrule
%     Neyshabur 2015 & 33.19   \\
%     Neyshabur 2017 &  30.10     \\
%     Bartlett 2017  &  22.40 \\
%     Ours     & 4.68 \\
%     True error & -2.41 \\
%     \bottomrule
%   \end{tabular}
%     \caption{Comparison of different generalization bounds on the CIFAR10 dataset on a logarithmic scale}
%     \label{table:cifar}
% \end{minipage}
% \begin{minipage}[b]{0.5\linewidth}
% \centering
% \includegraphics[width=50mm]{Figures/figure5.pdf}
% \captionof{figure}{Improvement of Generalization Bound over Training Process}
% \label{fig:overtraining}
% \end{minipage}
% \end{table}
%model size, we test on hidden dimensions 500, 1000, 1500, 2000, and 2500 where the depth is kept constant.
%Comparison of the generalization bounds on logarithmic scale w.r.t. (a) MNIST, and (b) CIFAR10 datasets.

\begin{figure}[H]
     \centering
         \includegraphics[width=.6\textwidth]{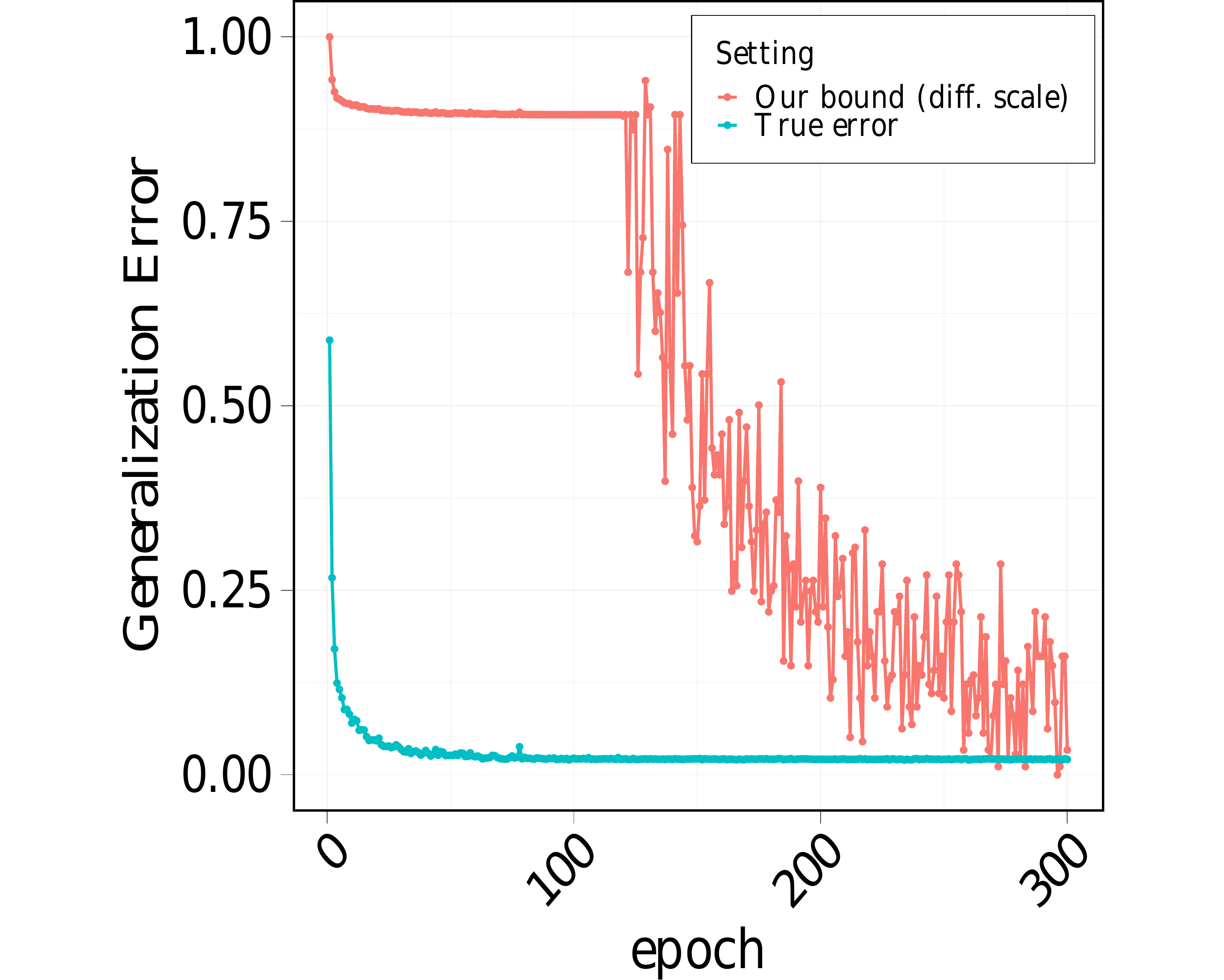}
         \caption{Comparing bounds on MNIST.}
         \label{fig:overtraining}
\end{figure}

\end{document}